\documentclass[reqno]{amsart}

\usepackage{amsmath, amsthm, amsfonts,amssymb}
\usepackage[square,sort,comma,numbers]{natbib}
\usepackage{graphicx, color}
\usepackage{subfigure}
\usepackage{epstopdf}
\usepackage{float}
\usepackage{graphicx}
\usepackage{amsmath}
\usepackage{amsfonts}
\usepackage{amssymb}
\usepackage{amsthm}
\usepackage{bm,bbm}
\usepackage[colorlinks=true,linkcolor=blue, citecolor=blue]{hyperref}

\theoremstyle{plain}
\newtheorem{thm}{Theorem}[section]

\newtheorem{lem}[thm]{Lemma}
\newtheorem{assu}[thm]{Assumption}

\newtheorem{defn}[thm]{Definition}

\theoremstyle{remark}
\newtheorem{rem}[thm]{Remark}

\newcommand{\xb}{\mathbf{x}}
\newcommand{\OO}{\mathrm{O}}
\newcommand{\Ob}{\mathbf{O}}
\newcommand{\sd}{\mathsf{d}}
\newcommand{\rSh}{\mathrm{Sh}}
\newcommand{\yb}{\mathbf{y}}
\newcommand{\zb}{\mathbf{z}}
\newcommand{\Kb}{\mathbf{K}}
\newcommand{\Eb}{\mathbf{E}}
\newcommand{\Lb}{\mathbf{L}}

\newcommand{\ub}{\mathbf{u}}
\newcommand{\wb}{\mathbf{w}}
\newcommand{\Wb}{\mathbf{W}}

\newcommand{\Rb}{\mathbf{R}}
\newcommand{\Hb}{\mathbf{H}}
\newcommand{\Tb}{\mathbf{T}}
\newcommand{\Ab}{\mathbf{A}}
\newcommand{\Ub}{\mathbf{U}}
\newcommand{\Db}{\mathbf{D}}
\newcommand{\Sb}{\mathbf{S}}
\newcommand{\Qb}{\mathbf{Q}}
\newcommand{\Zb}{\mathbf{Z}}
\newcommand{\Xb}{\mathbf{X}}

\newcommand{\Yb}{\mathbf{Y}}
\newcommand{\Pb}{\mathbf{P}}
\newcommand{\Nb}{\mathbf{N}}
\newcommand{\dd}{\mathrm{d}}

\newcommand{\Ib}{\mathbf{I}}

\date{}
\begin{document}
\title[]{How do kernel-based sensor fusion algorithms behave under high dimensional noise? }

\author{Xiucai Ding}
\address{Department of Statistics, University of California, Davis, CA, USA}
\email{xcading@ucdavis.edu}

\author{Hau-Tieng Wu}
\address{Department of Mathematics and Department of Statistical Science, Duke University, Durham, NC, USA}
\email{hauwu@math.duke.edu}

\maketitle

\begin{abstract} 
We study the behavior of two kernel based sensor fusion algorithms, nonparametric canonical correlation analysis (NCCA) and alternating diffusion (AD), under the nonnull setting that the clean datasets collected from two sensors are modeled by a common low dimensional manifold embedded in a high dimensional Euclidean space and the datasets are corrupted by high dimensional noise. We establish the asymptotic limits and convergence rates for the eigenvalues of the associated kernel matrices assuming that the sample dimension and sample size are comparably large, where NCCA and AD are conducted using the Gaussian kernel. 
It turns out that both the asymptotic limits and convergence rates depend on the signal-to-noise ratio (SNR) of each sensor and selected bandwidths. On one hand, we show that if NCCA and AD are directly applied to the noisy point clouds without any sanity check, it may generate artificial information that misleads scientists' interpretation.  On the other hand, we prove that if the bandwidths are selected adequately, both NCCA and AD can be made robust to high dimensional noise when the SNRs are relatively large. \end{abstract}

\section{Introduction}

How to adequately quantify the system of interest by assembling available information from multiple datasets collected simultaneously from different sensors is a long lasting and commonly encountered problem in data science. This problem is commonly referred to as the \emph{sensor fusion} problem \cite{6788402,7214350,gustafsson2012statistical,ZHAO201743}. While the simplest approach to ``fuse'' the information is via a simple concatenation of available information from each sensor, it is not the best and most efficient approach. To achieve a better and more efficient fusion algorithm, researchers usually face several challenges. For example, the sensors might be heterogeneous, datasets from different sensors might not be properly aligned, the datasets might be high dimensional and noisy, to name but a few. Roughly speaking, researchers are interested in extracting common components (information) shared by different sensors, if there is any, where roughly speaking.

A lot of effort was invested to find a satisfactory algorithm based on various models. Historically, when we can safely assume a linear structure in the common information shared by different sensors, the most typical algorithm in handling this problem is the canonical correlation analysis (CCA) \cite{hotelling1936relations} and its descendants \cite{Horst1961,6788402,Hwang2013}, which is a far from complete list. 
In the modern data analysis era, due to the advance of sensor development and growth of the complexities of problems, researchers may need to take the nonlinear structure of the datasets into account to better understand the datasets. 
To handle this nonlinear structure, several nonlinear sensor fusion algorithms have been developed, for example, nonparametric canonical correlation analysis (NCCA) \cite{NCCApaper}, alternative diffusion (AD) \cite{LEDERMAN2018509, TALMON2019848} and its generalization \cite{MR4010764}, time coupled diffusion maps \cite{marshall2018time}, multiview diffusion maps \cite{lindenbaum2020multi}, etc. 
See \cite{MR4010764} for a recent and more thoughtful list of available nonlinear tools and \cite{zhuang2020technical} for a recent review.
The main idea beyond these developments is that the nonlinear structure is modeled by various nonlinear geometric structures, and the algorithms are designed to preserve and capture this nonlinear structure.
Such ideas and algorithms have been successfully applied to many real world problems, like audio-visual voice activity detection \cite{8281539}, the study of the sequential audio-visual correspondence \cite{7558246}, automatic sleep stage annotation from two electroencephalogram signals \cite{liu2020diffuse}, seismic event modeling \cite{lindenbaum2018multiview}, fetal electrocardiogram analysis \cite{MR4010764} and IQ prediction from two fMRI paradigms \cite{xiao2019manifold}, which is a far from complete list.

While these kernel-based sensor fusion algorithms have been developed and applied for a while, there are still several gaps toward a solid practical application and sound theoretical understanding of these tools. One important gap is understanding how the inevitable noise, particularly when the data dimension is high, impacts the kernel-based sensor fusion algorithms. For example, can we be sure if the obtained fused information is really informative, particularly when the datasets are noisy or when one sensor is broken? If the signal-to-noise ratios of two sensors are different, how does these noises impact the information captured by these kernel based sensors? To our knowledge, the developed kernel-based sensor fusion algorithms do not take care of how the noise interacts with the algorithm, and most theoretical understandings are mainly based on the nonlinear data structure without considering the impact of high dimensional noise, except a recent effort in the null case \cite{DW1}. 
In this paper, we focus on one specific challenge among many; that is, we study how high dimensional noise impacts the spectrum of two kernel-based sensor fusion algorithms, NCCA and AD, in the {\em non-null setup} when there are two sensors. 

We briefly recall the NCCA and AD algorithms. Consider two noisy point clouds, $\mathcal{X}=\{\mathbf{x}_i\}_{i=1}^n\subset \mathbb{R}^{p_1}$ and $\mathcal{Y}=\{\mathbf{y}_j\}_{j=1}^n\subset \mathbb{R}^{p_2}$. 
For some bandwidths $h_1, h_2>0$ and some fixed constant $\upsilon>0$ chosen by the user, we consider two $n \times n$ \emph{affinity matrices}, $\Wb_1$ and $\Wb_2$, defined as 
\begin{equation}\label{eq_affinitymatrix}
\Wb_1(i,j)=\exp\left( -\upsilon \frac{\| \xb_i-\xb_j \|_2^2}{h_1}\right) \ \ \mbox{and} \ \  \Wb_2(i,j)=\exp\left( -\upsilon \frac{\| \yb_i-\yb_j \|_2^2}{h_2}\right),
\end{equation}
where $i,j=1,\ldots,n$. Here, $\Wb_1$ and $\Wb_2$ are related to the point clouds $\mathcal{X}$ and $\mathcal{Y}$ respectively. 
Denote the associated \emph{degree matrices} $\Db_1$ and $\Db_2,$ which are diagonal matrices such that 
\begin{equation}\label{eq_degreematrices}
\Db_1(i,i)=\sum_{j=1}^n \Wb_1(i,j)\ \ \mbox{ and }\ \ \Db_2(i,i)=\sum_{j=1}^n \Wb_2(i,j), \ i=1,2,\cdots,n\,. 
\end{equation}   
Moreover, denote the transition matrices $\Ab_1, \Ab_2$ as
\begin{equation*}
\Ab_1=\mathbf{D}_1^{-1} \Wb_1\ \ \mbox{ and }\ \  \Ab_2=\mathbf{D}_2^{-1} \Wb_2\,.
\end{equation*}
The NCCA and AD matrices are defined as 
\begin{equation}\label{eq_nccaadmatrixdefinition}
\mathbf{N}=\Ab_1 \Ab_2^\top\ \ \mbox{ and }\ \  \Ab=\Ab_1 \Ab_2
\end{equation}
respectively. Note that in the current paper, for simplicity, we focus our study on the Gaussian kernels. More general kernel functions will be our future topics. 
Usually, the top few eigenpairs of $\mathbf{N}$ and $\mathbf{A}$ are used as features of the extracted common information shared by two sensors. We shall emphasize that in general, while $\Ab_1$ and $\Ab_2$ are diagonalizable, $\mathbf{N}$ and $\mathbf{A}$ are not. But theoretically we can obtain the top few eigenpairs without a problem under the common manifold model \cite{TALMON2019848,MR4010764} since asymptotically $\mathbf{N}$ and $\mathbf{A}$ both converge to self-adjoint operators. To avoid this trouble, researchers also consider singular value decomposition (SVD) of $\mathbf{N}$ and $\mathbf{A}$.
Another important fact is that usually we are interested in the case when $\mathcal{X}$ and $\mathcal{Y}$ are {\em aligned}; that is, $\mathbf{x}_i$ and $\mathbf{y}_i$ are sampled from the same system at the same time. However, the algorithm can be applied to any two datasets of the same size, while it is not our concern in this paper.

\subsection{Some related works}

In this subsection, we summarize some related results. Since the NCCA and AD matrices (\ref{eq_nccaadmatrixdefinition}) are essentially products of transition matrices, we start from summarizing the results of the affinity and transition matrices when there is only one sensor. On one hand, in the noiseless setting, the spectral properties have been widely studied, for example, \cite{BELKIN20081289, 10.1007/11503415_32,MR2332434, SINGER2006128, MR4130541,dunson2019diffusion}, to name but a few.  In summary, under the manifold model, researchers show that the Graph Laplacian(GL) converges to the Laplace–Beltrami operator in various settings with properly chosen bandwidth. On other hand, the spectral properties have been investigated in \cite{CS,MR3044473,DV,NEKW2,NEKkernel,MR3916104,KR} under the null setup. 
These works essentially show that when $\mathcal{X}$ contains pure high-dimensional noise, the affinity and transition matrices are governed by a low-rank perturbed Gram matrix when the bandwidth $h_1=p_1$.
Despite rich literature above about two extreme setups, limited results are available in the intermittent, or nonnull, setup \cite{informationplusenoise,DW2,NEKW}. 
For example, when the signal-to-noise ratio (SNR), which will be defined precisely later, is sufficiently large, the spectral properties of GL constructed from the noisy observation are close to that constructed from the clean signal. Moreover, the bandwidth plays an important role in the nonnull setup. For a more comprehensive review and sophisticated study on the spectral properties of the affinity and transition matrices for an individual point cloud, we refer the readers to \cite[Sections 1.2 and 1.3]{DW2}.    

For the NCCA and AD matrices, on one hand, in the noiseless setting, there have been several results under the common manifold model \cite{lederman2018learning,talmon2019latent}. On the other hand, 
 under the null setup that both sensors only capture high dimensional white noise, its spectral property has been studied recently \cite{DW1}. Specifically, except for a few larger outliers, when $h_1=p_1$ and $h_2=p_2,$ the edge eigenvalues of $n^2 \Ab$ or $n^2 \Nb$ converge to some deterministic limit depending on the free convolution (c.f. Definition \ref{defn_freeandsubor}) of two Marchenko-Pastur (MP) laws \cite{MP}. However, in the nonnull setting when both sensors are contaminated by noise, to our knowledge, there does not exist any theoretical study, particularly under the high dimensional setup.

\subsection{An overview of our results}

We now provide an overview of our results. The main contribution of this paper is a comprehensive study of NCCA and AD under the non-null case in the high dimensional setup. This result can be viewed as a continuation of the study under the null case \cite{DW1}. We focus on the setup that the signal is modeled by a low dimensional manifold. It turns out that this problem can be recast as studying the algorithm under the commonly applied spiked model, which will be made clear later. In addition to providing a theoretical justification based on the kernel random matrix theory, we propose a  method to choose the bandwidth adaptively.
Moreover, peculiar and counterintuitive results will be presented when two sensors have different behavior, which emphasizes the importance of carefully applying these algorithms in practice. 
In Section \ref{sec_mainresults}, we investigate the eigenvalues of the NCCA and AD matrices when $h_1=p_1$ and $h_2=p_2$, which is a common choice in the literature. The behavior of the eigenvalues varies according to both SNRs of the point clouds. 
When both SNRs are small, the spectral behavior of 
$n^2 \Nb$ and $n^2 \Ab$ is like that in the null case, while  
both the number of outliers and the convergence rates rely on SNRs; see Theorem \ref{thm_senarioiionesubcritical} for details. 
Furthermore, if one of the sensors has large SNR and the other one has small SNR, 
the eigenvalues of $n\Nb$ and $n\Ab$ provide limited information about the signal; 
see Theorem \ref{thm_secenario2slowly} for details. We emphasize that this result warns us that if we directly apply NCCA and AD without any sanity check, it may result in a misleading conclusion. 
When both SNRs are larger, 
the eigenvalues are close to the clean NCCA and AD matrices; 
see Theorem \ref{thm_secenario2moderately} for more details. It is clear that the classic bandwidth choices $h_k=p_k$ for $k=1,2$ are inappropriate when the SNR is large, 
since the bandwidth is too small compared with the signal strength. In this case $\Nb \approx \Ib, \Ab \approx \Ib$, and we obtain limited information about the signal; see (\ref{eqeqeqeqqqqrqrqer}) for details. 
To handle this issue, in Section  \ref{sec_adpativebandwidth}, we consider bandwidths that are adaptively chosen according to the dataset. 
With this choice, when the SNRs are large, 
NCCA and AD become non-trivial and informative; that is, NCCA and AD are robust against the high dimensional noise. 
See Theorem \ref{thm_mainbandwidthtwo} for details. 
%

\vspace{6pt}

\noindent{\bf Conventions.} The fundamental large parameter is $n$ and we always assume that $p_1$ and $p_2$ are comparable to and depend on $n$. We use $C$ to denote a generic positive constant, and the value may change from one line to the next. Similarly, we use $\epsilon$, $\tau$, $\delta$, etc., to denote generic small positive constants. If a constant depends on a quantity $a$, we use $C(a)$ or $C_a$ to indicate this dependence. For two quantities $a_n$ and $b_n$ depending on $n$, the notation $a_n = \OO(b_n)$ means that $|a_n| \le C|b_n|$ for some constant $C>0$, and $a_n=\mathrm{o}(b_n)$ means that $|a_n| \le c_n |b_n|$ for some positive sequence $c_n\downarrow 0$ as $n\to \infty$. We also use the notations $a_n \lesssim b_n$ if $a_n = \OO(b_n)$, and $a_n \asymp b_n$ if $a_n = \OO(b_n)$ and $b_n = \OO(a_n)$.
For a $n\times n$ matrix $A$,  $\|A\|$ indicates the operator norm of $A$, and $A=O(a_n)$ means $\|A\|\leq Ca_n$ for some constant $C>0$.
Finally, for a random vector $\ub,$ we say it is sub-Gaussian if for any deterministic vector $\mathbf{a}$, we have $\mathbb{E}(\exp(\mathbf{a}^\top \mathbf{u})) \leq \exp(\| \mathbf{a}\|_2^2/2)$. 

The paper is organized as follows. In Section \ref{sec_mathrmtbackground}, we introduce the mathematical setup and some background in random matrix theory. In Section \ref{sec_mainresults}, we state our main results for the classic choice of bandwidth. In Section \ref{sec_adpativebandwidth}, we state the main results for the adaptively chosen bandwidth.  In Section \ref{sec_proofs}, we offer the technical proofs of the main results. In Appendix \ref{sec_auxilemma}, we provide and prove some preliminary results which will be used in the technical proofs.

\section{Mathematical framework and background}\label{sec_mathrmtbackground}

\subsection{Mathematical framework}

We focus on the following model for the datasets $\mathcal{X}$ and $\mathcal{Y}$. Assume that the first sensor i.i.d. sample $n$ clean signals from a sub-Gaussian vector $X:(\Omega, \mathcal{F},\mathbb{P})\to \mathbb{R}^{p_1}$, denoted as $\{\mathbf{u}_{ix} \}_{i=1}^n \in \mathbb{R}^{p_1}$, where $(\Omega, \mathcal{F},\mathbb{P})$ is a probability space. Similarly, assume that the second sensor also i.i.d. sample clean signals from a sub-Gaussian vector $Y$, denoted as $\{\mathbf{u}_{iy}\}_{i=1}^n \in \mathbb{R}^{p_2}$. 
Since we focus on the distance of pairwise samples, without loss of generality, we assume that 
\begin{equation}\label{eq_signalmeanzeroassumption}
\mathbb{E}\ub_{ix}=\mathbf{0}\quad\mbox{and}\quad\mathbb{E}\ub_{iy}=\mathbf{0}.
\end{equation}
Denote $S_1=\operatorname{Cov}(\ub_{ix})$ and $S_2=\operatorname{Cov}(\ub_{iy})$, and to simplify the discussion, we assume that $S_1$ and $S_2$  admit the following spectral decomposition 
\begin{equation}\label{eq_spectraldecompositions1s2}
S_k=\operatorname{diag}\{\sigma_{k1}^2, \cdots, \sigma_{kd}^2, 0, \cdots, 0\}\in \mathbb{R}^{p_k\times p_k}, \ k=1,2,  
\end{equation}
where $d_1$ and $d_2$ are fixed integers. We model the {\em common information} by assuming that there exists a bijection $\phi$ so that 
\begin{equation}\label{definition phi between X and Y}
X(\Omega)=\phi(Y(\Omega))\,;
\end{equation}  
that is, we have $\ub_{ix}=\phi(\ub_{iy})$ for any $i=1,\ldots,n$.
In practice, the clean signals $\{\mathbf{u}_{ix} \} $ and $\{\mathbf{u}_{iy}\}$ are contaminated by two sequences of i.i.d. sub-Gaussian noise $\{\zb_i\} \in \mathbb{R}^{p_1}$ and $\{\wb_i\} \in \mathbb{R}^{p_2}$, respectively, so that the data generating process follows 
\begin{equation}\label{eq_basicmodel}
\mathbf{x}_i= \mathbf{u}_{ix}+\mathbf{z}_i \ \ \mbox{and}\ \ \mathbf{y}_i= \mathbf{u}_{iy}+\mathbf{w}_i\,,
\end{equation}  
where
\begin{equation}\label{eq_s2noisemodel11}
\mathbb{E} \mathbf{z}_i=\mathbf{0}_{p_1}, \ \operatorname{Cov}(\zb_i)=\mathbf{I}_{p_1},  \  \mathbb{E} \mathbf{w}_i=\mathbf{0}_{p_2}, \ \operatorname{Cov}(\wb_i)=\mathbf{I}_{p_2}.
\end{equation}
We further assume that $\mathbf{z}_i$ and $\mathbf{w}_i$ are independent with each other and also independent of  $\{\ub_{ix}\}$ and $\{\ub_{iy}\}$. 
We are mainly interested in the {\em high dimensional} setting; that is, $p_1$ and $p_2$ are comparably as large as $n.$ More specifically, we assume that there exists some small constant $0<\gamma<1$ such that
\begin{equation}\label{assumptuon_ratio}
\gamma \leq c_1:=\frac{n}{p_1} \leq \gamma^{-1}\ \ \mbox{and}\ \ \gamma\leq c_2:=\frac{n}{p_2} \leq \gamma^{-1}. \ 
\end{equation}
The SNRs in our setting are defined as $\{\sigma_{1i}^2\}$ and $\{\sigma_{2i}^2\},$ respectively, so that for all $1 \leq i \leq d_1$ and $1 \leq j\leq d_2$,
\begin{equation}\label{eq_variancesnrration1}
\sigma_{1i}^2 \asymp n^{\zeta_{1i}}, \ \sigma_{2j}^2 \asymp n^{\zeta_{2j}}\,,
\end{equation}
for some constants $\zeta_{1i},\zeta_{2j} \geq 0$.  To avoid repetitions, we summarize the assumptions as follows.
\begin{assu}\label{assu_main} Throughout the paper, we assume that (\ref{eq_signalmeanzeroassumption})--(\ref{eq_variancesnrration1}) hold.
\end{assu}

In view of (\ref{eq_spectraldecompositions1s2}), the model (\ref{eq_basicmodel}) for each sensor is related to the spiked covariance matrix models \cite{johnstone2001}. 
%
%
%
We comment that this seemingly simple model, particularly \eqref{eq_spectraldecompositions1s2}, includes the commonly considered nonlinear {\em common manifold model}. 
In the literature, the common manifold model means that two sensors sample simultaneously from {\em one} low dimensional manifold; that is, $\mathbf{u}_{ix} =\mathbf{u}_{iy} \in M$ and $\phi$ is an identity map, where $M$ is a low dimensional smooth and compact manifold embedded in the high dimensional Euclidean space.  Since we are interested in the kernel matrices depending on pairwise distances, which is invariant to rotation, when combined with Nash's embedding theory, the common manifold can be assumed to be supported in the first few axes of the high dimensional space, like that in \eqref{eq_spectraldecompositions1s2}. As a result, the common manifold model becomes a special case of the model \eqref{eq_basicmodel}. We refer readers to \cite{DW2} for a detailed discussion of this relationship. A special example of the common manifold model is the widely considered linear subspace as the common component; that is, when $M=\mathbb{R}^{d}$ embedded in $\mathbb{R}^{p_k}$ for $k=1,2$. In this case, we could simply apply CCA to estimate the common component, and its behavior in the high dimensional setup has been studied in \cite{BHPZ,2021arXiv210203297M}. 

We should emphasize that through the analysis of NCCA and AD under the common component model satisfying Assumption \eqref{assu_main}, we do not claim that we could understand the underlying manifold structure. The problem we are asking here is the nontrivial relationship between the noisy and clean affinity and transition matrices, while the problem of exploring the manifold structure from the {\em clean} datasets \cite{MR4130541,dunson2019diffusion} is a different one, which is usually understood as the {\em manifold learning} problem. To study the nontrivial relationship between the noisy and clean affinity and transition matrices, it is the spiked covariance structure that we focus on, but not the possibly non-trivial $\phi$.
By establishing the nontrivial relationship between the noisy and clean affinity and transition matrices in this paper, when combined with the knowledge of manifold learning via the kernel-based manifold learning algorithm with clean datasets \cite{LEDERMAN2018509,talmon2019latent,shnitzer2019recovering}, we know how to explore the common manifold structure, which depends on $\phi$, from the {\em noisy} datasets.

\begin{rem}
While it is not our focus in this paper, we should mention that our model includes the case that the datasets captured by two sensors are not exactly on {\em one} manifold $M$, but from {\em two} manifolds that are diffeomorphic to $M$ \cite{shnitzer2019recovering}. Specifically, the first sensor samples points $\{\mathbf{u}_{ix} \}$ from $\phi_1(M)$, while the second sensor simultaneously samples points $\{\mathbf{u}_{iy} \}$ from $\phi_2(M)$, where $\phi_1$ and $\phi_2$ are both diffeomorphisms and $\mathbf{u}_{iy}=\phi_1(\phi_2^{-1}(\mathbf{u}_{ix}))$; that is, $\phi=\phi_1\circ \phi_2^{-1}$ in \eqref{definition phi between X and Y}. Note that in this case, $d_1$ might be different from $d_2$.
Moreover, the samples from two sensors can be more general. For example, in \cite{talmon2019latent}, the principle bundle structure is considered to model the ``nuisance'', which can be understood as the ``deterministic noise'', and in \cite{LEDERMAN2018509} the metric space as the common component is considered. 
While it is possible to consider a more complicated one, since we are interested in studying how noise impacts NCCA and AD, in this paper we simply focus on the above model but not further elaborate this possible extension. 
\end{rem}

\subsection{Some random matrix theory background}

In this subsection, we introduce some random matrix theory background and necessary notations. Let $\Zb \in \mathbb{R}^{p_1 \times n}$ be the data matrix associated with $\{\zb_i\}$; that is, the $i$-th column $\Zb$ is $\zb_i$, and consider the scaled noise $s\Zb$, where $s>0$ stands for the standard deviation of the scaled noise. Denote the empirical spectral distribution (ESD) of $\Qb=\frac{s^2}{p_1}\Zb^\top \Zb$ as 
\begin{equation*}
\mu_{\Qb}(x)=\frac{1}{n} \sum_{i=1}^n \mathbf{1}_{\{\lambda_i(\Qb)\leq x\}}, \ x \in \mathbb{R}. 
\end{equation*}
It is well-known that in the high dimensional regime (\ref{assumptuon_ratio}), $\mu_{\Qb}$ has the same asymptotic \cite{MR3704770} as the so-called MP law \cite{MP}, denoted as $\nu_{c_1,s^2}$, satisfying
\begin{equation}\label{eq_mp}
\nu_{c_1,s^2}(I)=
(1-c_1)_{+} \chi_{I}(0)+\zeta_{c_1,s^2}(I)\,,
\end{equation}
where $I \subset \mathbb{R}$ is a measurable set, $\chi_I$ is the indicator function and $(a)_+:=0$ when $a\leq 0$ and $(a)_+:=a$ when $a>0$, 
\begin{equation}\label{eq_densitymplaw}
d \zeta_{c_1,s^2}(x)=\frac{1}{2\pi s^2} \frac{\sqrt{(\lambda_{+,1}-x)(x-\lambda_{-,1})}}{c_nx} \mathrm{d}x\,, 
\end{equation}
$\lambda_{+,1}=(1+ s^2 \sqrt{c_1})^2$ and $\lambda_{-,1}=(1- s^2 \sqrt{c_1})^2$. 
Denote
\begin{equation}\label{eq_defntauk}
\tau_1 \equiv \tau_1(\lambda_1):=2 \left(\frac{\lambda_1}{p_1}+1 \right), \  \tau_2 \equiv \tau_2(\lambda_2):=2 \left(\frac{\lambda_2}{p_2}+1 \right),
\end{equation}
and for $k=1,2$,
\begin{equation}\label{eq_defnvarsigmak}
\varsigma_k \equiv \varsigma_k(\tau_k):=1-2 \upsilon \exp(-\upsilon \tau_k)-\exp(-\upsilon \tau_k)\,.  
\end{equation}
For any constant $\mathsf{a}>0,$ denote $\mathrm{T}_{\mathsf{a}}$ be the shifting operator that shifts a probability measure $\nu$ defined on $\mathbb{R}$ by $\mathsf{a};$ that is 
\begin{equation}\label{eq_defnshiftoperator}
\mathrm{T}_{\mathsf{a}} \nu(I)=\nu(I-\mathsf{a}),
\end{equation}
where $I-\mathsf{a}$ means the shifted set. Using the notation (\ref{eq_mp}), for $k=1,2,$ denote 
\begin{equation}\label{eq_nu1nu2}
\nu_k:= \mathrm{T}_{\varsigma_k} \nu_{c_k, \eta}, \ \mbox{where }  \ \eta=2 \upsilon \exp(-2 \upsilon). 
\end{equation} 
Next, we introduce a $n$-dependent quantity of some probability measure.  For a given probability measure $\mu$ and $n \in \mathbb{N},$ define $\gamma_\mu(j)$ as 
\begin{equation}\label{eq_defnclassiciallocaltion}
\int_{\gamma_\mu(j)}^\infty \mu(\dd x)=\frac{j}{n}. 
\end{equation}
Finally, we recall the following notion of \emph{stochastic domination} \cite[Chapter 6.3]{erdos2017dynamical} that we will frequently use. 
Let
$ \mathsf{X}=\big\{\mathsf{X}^{(n)}(u):  n \in \mathbb{N}, \ u \in \mathsf{U}^{(n)}\big\}$ and $\mathsf{Y}=\big\{\mathsf{Y}^{(n)}(u):  n \in \mathbb{N}, \ u \in \mathsf{U}^{(n)}\big\}$
be two families of nonnegative random variables, where $\mathsf{U}^{(n)}$ is a possibly $n$-dependent parameter set. We say that $\mathsf{X}$ is {\em stochastically dominated} by $\mathsf{Y}$, uniformly in the parameter $u$, if for any small $\upsilon>0$ and large $ D>0$, there exists $n_0(\upsilon, D)\in \mathbb{N}$ so that we have $\sup_{u \in \mathsf{U}^{(n)}} \mathbb{P} \Big( \mathsf{X}^{(n)}(u)>n^{\upsilon}\mathsf{Y}^{(n)}(u) \Big) \leq n^{- D}$, for a sufficiently large $n \geq  n_0(\upsilon, D)$. We interchangeably use the notation $\mathsf{X}=\OO_{\prec}(\mathsf{Y})$ {or $\mathsf{X} \prec \mathsf{Y}$}  if $\mathsf{X}$ is stochastically dominated by $\mathsf{Y}$, uniformly in $u$, when there is no danger of confusion.
In addition, we say that an $n$-dependent event $\Omega \equiv \Omega(n)$ holds {\em with high probability} if for a $D>1$, there exists $n_0=n_0(D)>0$ so that
$\mathbb{P}(\Omega) \geq 1-n^{-D}$, when $n \geq n_0.$

\subsection{A brief summary of free multiplication of random matrices}\label{sec_appendixfpt}


In this subsection, we summarize some preliminary results about free multiplication of random matrices from \cite{DJ, Jipaper}. Given some probability measure $\mu,$ its Stieltjes transform and $M$-transform are defined as 
\begin{equation*}
m_{\mu}(z)=\int \frac{1}{x-z} \dd \mu(x) \ \ \mbox{and}\ \ M_{\mu(z)}=\frac{zm_{\mu}(z)}{1+zm_{\mu}(z)}\,,
\end{equation*}
where $z \in \mathbb{C} \backslash \mathbb{R}_+$, respectively.
We next introduce the subordination functions  utilizing the $M$-transform \cite{Jipaper,DANORIGINALPAPER}. 
For any two  probability measures $\mu_x$ and $\mu_y$, there exist analytic functions  $\Omega_x(z)$ and $\Omega_y(z)$ satisfying 
\begin{equation}\label{eq_subordinationequation}
zM_{\mu_x}(\Omega_y(z))=z M_{\mu_y}(\Omega_x(z))=\Omega_x(z) \Omega_y(z) , \ \text{for} \ z \in \mathbb{C} \backslash \mathbb{R}_+. 
\end{equation}
Armed with the subordination functions, we now introduce the free multiplicative convolution of $\mu_x$ and $\mu_y,$ denoted as $\mu_x \boxtimes \mu_y$, when $\mu_x$ and $\mu_y$ are compactly supported on $\mathbb{R}_+$ but not both delta measures supported on $\{0\}$; see Definition 2.7 of \cite{DJ}. 
\begin{defn}\label{defn_freeandsubor} Denote the analytic function $M$ by 
\begin{equation}\label{eq_fmcdefinition}
M(z):=M_{\mu_x}(\Omega_y(z))= M_{\mu_y}(\Omega_x(z)).
\end{equation}
Then the free multiplicative convolution $\mu_x \boxtimes \mu_y$ is defined as the unique probability measure that (\ref{eq_fmcdefinition}) holds for all $z \in \mathbb{C} \backslash \mathbb{R}_+,$ i.e., $M(z) \equiv M_{\mu_x \boxtimes \mu_y}(z)$ is the $M$-transform of $\mu_x \boxtimes \mu_y.$  Moreover, $\Omega_x$ and $\Omega_y$ are referred to as the subordination functions.  
\end{defn}

For $\nu_1$ and $\nu_2$ defined in (\ref{eq_nu1nu2}), we have two sequences $a_k=\gamma_{\nu_1}(k-1)$ and $b_k=\gamma_{\nu_2}(k-1)$, where $1 \leq k \leq n$. Note that we have 
\begin{equation}\label{eq_rigiditydefinition}
\int_{a_k}^{E_{+,1}} \dd \nu_1(x)=\frac{k-1}{n},  \ \int_{b_k}^{E_{+,2}} \dd \nu_2(x)=\frac{k-1}{n}, 
\end{equation}
where $E_{+,1}, E_{+,2}$ are the right edges of $\nu_1$ and $\nu_2,$ respectively. Denote two $n \times n$ positive definite matrices $\Sigma_1$ and $\Sigma_2$ as follows 
\begin{equation}\label{eq_definsigma1sigma2}
\Sigma_1=\operatorname{diag}\{a_1, a_2, \cdots, a_n\}, \ \Sigma_2=\operatorname{diag}\{b_1, b_2, \cdots, b_n\}.
\end{equation}
Let $\Ub$ be an $n \times n$ Haar distributed random matrix in $O(n)$ and denote 
\begin{equation*}
\Hb=\Sigma_2 \Ub \Sigma_1 \Ub^\top. 
\end{equation*} 
The following lemma summarizes the rigidity of eigenvalues of $\Hb.$ 
\begin{lem}\label{lem_multiplication}
Suppose (\ref{eq_definsigma1sigma2}) holds. Then we have that
\begin{equation*}
\sup_j \left| \lambda_j(\Hb)-\gamma_{\nu_{1} \boxtimes \nu_2}(j) \right| \prec n^{-2/3} \widetilde{j}^{-1/3}, \ \widetilde{j}:=\min\left\{p_1 \wedge n+1-j,j\right\}.
\end{equation*} 
\end{lem}
\begin{proof}
The proof follows from Theorems 2.14 and 2.20 of \cite{DJ} since the Assumptions 2.2, 2.4 and 2.7 in \cite{DJ} are satisfied. Especially, (iii) of Assumption 2.2 in \cite{DJ} holds due to the square root behavior of the MP laws as indicated by (\ref{eq_densitymplaw}). 
\end{proof}

\section{Main results (I)--classic bandwidth: $h_1 \asymp p_1, h_2 \asymp p_2$ }\label{sec_mainresults}

In this section, we state our main results regarding the  eigenvalues of $\mathbf{N}$ and $\mathbf{A}$ when $h_k \asymp p_k$, where $k=1,2.$ For definiteness, we assume that  
 $h_1=p_1$ and $h_2=p_2.$ In what follows, for the ease of statements, we focus on reporting the results for $d=1$ and hence omit the subscripts of the indices $i,j$ in (\ref{eq_variancesnrration1}). For the general setting with $d_1>1$ or $d_2>1$, we refer the readers to Remark \ref{rmk_generalspike} below for more details. Finally, we focus on reporting the results for the NCCA matrix $\Nb$. The results for the AD matrix $\Ab$ are similar. For the details of the AD matrix $\Ab$, we refer the readers to Remark  \ref{rmk_admatrix} below.
Moreover, by symmetry, without loss of generality, we always assume that $\zeta_1 \geq \zeta_2;$ that is, the first sensor always has a larger SNR.

\subsection{Noninformative region: $0 \leq \zeta_2<1$}\label{sec_s2}
In this subsection, we state the results when at least one sensor contains strong noise, or equivalently has a small SNR; that is, $0 \leq \zeta_2<1$. In this case, the NCCA and AD will not be able to provide useful information or can only provide limited information for the underlying common manifold.

\subsubsection{When both sensors have small SNRs, $0 \leq \zeta_2 \leq \zeta_1 < 1$}
In this case, both sensors have small SNRs such that the noise dominates the signal. For some fixed integers $\mathsf{s}_1, \mathsf{s}_2$ satisfying
\begin{equation}\label{eq_defns1s2}
4 \leq \mathsf{s}_k \leq C 4^{\mathfrak{d}_k}, \ \mathfrak{d}_k:= \left \lceil  \frac{1}{1-{\zeta_k}} \right \rceil+1, \ k=1,2,  
\end{equation}
where $C>0$ is some constant, denote $\mathsf{T}$ as
\begin{equation}\label{eq_defnmathsfT}
\mathsf{T}:=
\begin{cases}
8, & \ 0 \leq \zeta_2 \leq \zeta_1<0.5, \\
\mathsf{s}_1+8, & \ 0 \leq \zeta_2<0.5 \leq \zeta_1 <1, \\
\mathsf{s}_1+\mathsf{s}_2+8, & \ 0.5 \leq \zeta_2 \leq \zeta_1<1.  
\end{cases}
\end{equation}
Moreover, define $\mathsf{e}_k<0, k=1,2,$ as 
\begin{equation}\label{Definition: zeta in Scenario II}
\mathsf{e}_k:=(\zeta_k-1)\left(\left\lceil \frac{1}{1-\zeta_k}\right\rceil+1 \right)+1.
\end{equation}

\begin{thm}\label{thm_senarioiionesubcritical}
Suppose Assumption \ref{assu_main} holds with $0 \leq \zeta_2 \leq \zeta_1<1$, $h_1=p_1$, $h_2=p_2$ and $d_1=d_2=1.$ Moreover, we assume that
\begin{equation}\label{eq_s2noisemodel}
\mathbf{z}_i \overset{\operatorname{i.i.d.}}{\sim} \mathcal{N}(0, \mathbf{I}_{p_1}),  \  \mathbf{w}_i \overset{\operatorname{i.i.d.}}{\sim} \mathcal{N}(0, \mathbf{I}_{p_2})\, ,
\end{equation}
and there exists some constant $\tau>0$ such that 
\begin{equation}\label{eq_ratioone}
|c_k-1| \geq \tau, \ k=1,2. 
\end{equation}
Then, when $n$ is sufficiently large, we have that for $i>\mathsf{T}$ in (\ref{eq_defnmathsfT}),
\begin{equation}\label{eq_finalresultthm31}
\left| \lambda_i(n^2 \Nb)-\exp(4 \upsilon) \gamma_{\nu_1 \boxtimes \nu_2}(j) \right|=\OO_{\prec}\left( \max \left\{ n^{\frac{\zeta_1-1}{2}}, n^{\mathsf{e}_1}\right\}  \right), 
\end{equation} 
where $\mathsf{e}_1$ is defined in (\ref{Definition: zeta in Scenario II}).  
\end{thm}

Intuitively, in this region we cannot obtain any information about the signal, since asymptotically the noise dominates the signal. In practice, the datasets might fall in this region when both sensors are corrupted or the environment noise is too strong. This intuition is confirmed by Theorem \ref{thm_senarioiionesubcritical}. 
As discussed in \cite{DW2,NEKkernel}, when the noise dominates the signal, the outlier eigenvalues are mainly from the kernel function expansion or the Gram matrix and hence are not useful to study the underlying manifold structure. The number of these outlier eigenvalues depend on the SNR as can be seen in (\ref{eq_defnmathsfT}), which can be figured out from the kernel function expansion.

We should point out that (\ref{eq_s2noisemodel}) and (\ref{eq_ratioone}) are mainly technical assumptions and commonly used in the random matrix theory literature.  They guarantee that the individual bulk eigenvalues of $\Nb$ can be characterized by the quantiles of free multiplicative convolution.  
Specifically, (\ref{eq_ratioone}) ensures that the Gram matrices are bounded from below and (\ref{eq_s2noisemodel}) has been used in \cite{DW1} to ensure that the eigenvectors of the Gram  matrix are Haar distributed. As is discussed in \cite{DW1}, while it is widely accepted that the eigenvectors of the Gram matrix from i.i.d. sub-Gaussian random vectors are Haar distributed, we cannot find a proper proof. Since the proof of this Theorem depends on the results in \cite{DW1}, we impose the same condition. The assumption (\ref{eq_s2noisemodel}) can be removed when we can show that the eigenvectors of the Gram matrix from i.i.d. sub-Gaussian random vectors are Haar distributed. 
Since this is not the focus of the current paper, we will pursue this direction in future works.   

Finally, we mention that Theorem \ref{thm_senarioiionesubcritical} holds for more general kernel function beyond the Gaussian kernel. For example, as discussed in \cite[Remark 2.4]{DW2}, we can choose a general kernel function which is decreasing, $C^3$ and $f(2)>0.$

\subsubsection{When one sensor has a small SNR, $0 \leq \zeta_2<1 \leq \zeta_1<\infty$}

In Theorem \ref{thm_secenario2slowly} below, we consider that $0 \leq \zeta_2<1 \leq \zeta_1<\infty,$ i.e., one of the sensors has a large SNR whereas the other is dominated by the noise.  We prepare some notations here. Let $\Wb_{1,s}$ and $\Wb_{2,s}$ be the affinity matrices associated with $\{\ub_{ix}\}$ and $\{\ub_{iy}\}$ respectively, where the subscript $s$ stands for the short-hand notation for the signal. In other words, $\Wb_{1,s}$ and $\Wb_{2,s}$ are constructed from the clean signal. In general, since $h_1$ may be different from $h_2$, $\Wb_{1,s}$ and $\Wb_{2,s}$ might be different. Denote
\begin{align}\label{eq_widetildewxss2}
\widetilde{\Wb}_{1,s} =\exp(-2  \upsilon) \Wb_{1,s}+(1-\exp(-2   \upsilon))\Ib. 
\end{align}
Analogously, we denote the associated degree matrix and transition matrix as $\widetilde{\Db}_{1,s}$ and $\widetilde{\Ab}_{1,s}$ respectively, that is,
\begin{equation}\label{eq_tildeaxs2}
\widetilde{\Ab}_{1,s}=\widetilde{\Db}_{1,s}^{-1} \widetilde{\Wb}_{1,s}. 
\end{equation}
Define $\widetilde{\Wb}_{2,s}$ and $\widetilde{\Ab}_{2,s}$ similarly. 
Note that from the random walk perspective, $\widetilde{\Ab}_{1,s}$ (and $\widetilde{\Ab}_{2,s}$ as well) describe a lazy random walk on the clean dataset.
We further introduce some other $n\times n$ matrices, 
\begin{equation*}
\Wb_{1,c}(i,j)=\exp\left(-2 \upsilon \frac{(\ub_{ix}-\ub_{jx})^\top (\zb_i-\zb_j)}{p_1} \right)
\quad\mbox{and}\quad
\widetilde{\Wb}_{1,c}:=\widetilde{\Wb}_{1,s} \circ \Wb_{1,c}\,.
\end{equation*}
We then define the associate degree matrix and transition matrix as $\widetilde{\Db}_{1,c}$ and $\widetilde{\Ab}_{1,c},$ respectively; that is, 
\begin{equation}\label{eq_crossingtermdefinition}
\widetilde{\Ab}_{1,c}=\widetilde{\Db}_{1,c}^{-1} \widetilde{\Wb}_{1,c}\,.
\end{equation}
$\widetilde{\Wb}_{1,c}$ and $\widetilde{\Ab}_{1,c}$ will be used when $\zeta_1$ is too large ($\zeta_1 \geq 2$) so that the bandwidth $h_1 = p_1$ is insufficient to capture the relationship between two different samples.

With the above notations and (\ref{eq_defnvarsigmak}), denote 
\begin{equation}\label{eq_defnM1M2}
\widetilde{\Nb}:=
\begin{cases}
\exp(2 \upsilon)\widetilde{\Ab}_{1,s} \left(\varsigma_2 \Ib+2 \frac{\upsilon \exp(-\upsilon \tau_2)}{p_2}\Wb^\top \Wb \right), & 1 \leq \zeta_1<2; \\
\exp(2 \upsilon)\widetilde{\Ab}_{1,c} \left(\varsigma_2 \Ib+2 \frac{\upsilon \exp(-\upsilon \tau_2)}{p_2}\Wb^\top \Wb \right), & \zeta_1 \geq 2. 
\end{cases}
\end{equation}
Using $\mathsf{s}_2$ defined in (\ref{eq_defns1s2}), denote 
\begin{equation}\label{eq_defns}
\mathsf{S}:=
\begin{cases}
4, & 0 \leq \zeta_2<0.5, \\
\mathsf{s}_2+4, & 0.5 \leq \zeta_2<1.  
\end{cases}
\end{equation}

\begin{thm}\label{thm_secenario2slowly} 
Suppose Assumption \ref{assu_main} holds with $0 \leq \zeta_2<1 \leq \zeta_1<\infty$, $h_1=p_1$, $h_2=p_2$ and $d_1=d_2=1$.  Then we have that for $i > \mathsf{S}$,
\begin{equation}\label{eq_firstpartthm33}
\left| \lambda_i(n \Nb)-\lambda_i(\widetilde{\Nb}) \right| =\OO_{\prec}\left(\max\left\{ n^{\mathsf{e}_2}, n^{\frac{\zeta_2-1}{2}} \right\} \right),
\end{equation}
where $\mathsf{e}_2$ is defined in (\ref{Definition: zeta in Scenario II}) and $\widetilde{\Nb}$ is defined in (\ref{eq_defnM1M2}).  Furthermore, when {$\zeta_1$} is larger in the sense that  for any given small constant $\delta \in (0,1),$
\begin{equation}\label{eq_largeralphadefinition}
\zeta_1>\frac{2}{\delta}+1,
\end{equation}
then with probability at least $1-\OO\left(n^{1-\delta(\zeta_1-1)/2} \right),$ for some sufficiently small constant $\epsilon>0$ and some constant $C>0$ and all $i\geq \mathsf{S},$ we have   
\begin{equation}\label{eq_secondpartthm33}
\left| \lambda_i(n \Nb)-\exp(2  \upsilon) \gamma_{\nu_2}(i)  \right| \leq  C \max\{ n^{-1/2+\epsilon} ,n^{\mathsf{e}_2}, \ n^{\frac{{\zeta_2}-1}{2}}\}. 
\end{equation}   
\end{thm}

This is a potentially confusing region. In practice, it captures the situation when one sensor is corrupted so that the signal part becomes weak. Since we still have one sensor available with a strong SNR, it is expected that we could still obtain something useful. However, it is shown in
Theorem \ref{thm_secenario2slowly} that the corrupted sensor unfortunately contaminates the overall performance of the sensor fusion algorithm.
%
Note that since the first sensor has a large SNR, the noisy transition matrix $\Ab_1$ is close to the transition matrix $\widetilde{\Ab}_{1,s}$, which only depends on the signal part when $1 \leq \zeta_1<2,$ and the transition $\widetilde{\Ab}_{1,c}$ which is a mixture of the signal and noise when $\zeta_1 \geq 2$. This fact has been shown in \cite{DW2}.  However, for the second sensor, due to the strong noise, $\Ab_2$ will be close to a perturbed Gram matrix that mainly comes from the high dimensional noise. Consequently, as illustrated in (\ref{eq_firstpartthm33}), the NCCA matrix will be close to $\widetilde{\Nb}$ which is a product matrix of the clean transition matrix and the shifted Gram matrix. Clearly, the clean transition matrix is contaminated by the shifted Gram matrix, which does not contain any information about the signal. This limits the information we can obtain.

In the extreme case when $\zeta_1$ is larger in the sense of (\ref{eq_largeralphadefinition}), the chosen bandwidth $h_1=p_1$ is too small compared with the signal so that the transition matrix  $\Ab_1$ will be close to the identity matrix. Consequently, as in (\ref{eq_secondpartthm33}), the NCCA matrix will be mainly characterized by the perturbed Gram matrix whose limiting ESD follows the MP law $\nu_2$ with proper scaling.

We should however emphasize that as has been elaborated in \cite{DW2}, when the SNR is large, particularly when $\zeta_1>2$, we should consider a different bandwidth, particularly the bandwidth determined by the percentile of pairwise distance that is commonly considered in practice. It is thus natural to ask if the bandwidth $h_1$ is chosen ``properly'', would we obtain useful information eventually. We will answer this question in the later section.

\subsection{Informative region: $\zeta_2 \geq 1$}

In this subsection,  we state the results when both of the sensors have large SNR ($\zeta_2\geq 1$). Recall (\ref{eq_tildeaxs2}), (\ref{eq_crossingtermdefinition}) and denote $\widetilde{\Ab}_{2,s}, \widetilde{\Ab}_{2,c}$ analogously for the point cloud $\mathcal{Y}.$ For some constant $C>0,$ denote
\begin{equation}\label{eq_defnRindex}
\mathsf{R}:=
\begin{cases}
C \log n, &  \zeta_2=1\\
Cn^{\zeta_2-1}, & 1 <\zeta_2<2.
\end{cases}
\end{equation}

\begin{thm}\label{thm_secenario2moderately}  Suppose Assumption \ref{assu_main} holds with $1 \leq \zeta_2 \leq \zeta_1<\infty$, $h_1=p_1$, $h_2=p_2$ and $d_1=d_2=1$. Then we have that:
\begin{enumerate}
\item[(1).] When $1 \leq \zeta_2\leq \zeta_1 <2,$ we have 
\begin{equation}\label{eq_thm35eqone}
\left\| \Nb-\widetilde{\Ab}_{1,s} \widetilde{\Ab}_{2,s}^\top \right\| \prec n^{-1/2}. 
\end{equation}
Additionally, for $i>\mathsf{R}$, we have
\begin{equation}\label{eq_thm35eqtwo}
\lambda_i(\widetilde{\Ab}_{1,s} \widetilde{\Ab}_{2,s}^\top ) \prec n^{(\zeta_2-3)/2}. 
\end{equation}
\item[(2).] When $1 \leq \zeta_2<2 \leq \zeta_1<\infty,$ we have that (\ref{eq_thm35eqtwo}) holds true and 
\begin{equation}\label{eq_1leqzeta2leq2leqzetaone}
\| \Nb-\widetilde{\Ab}_{1,c} \widetilde{\Ab}_{2,s}^\top \| \prec n^{-1/2}.
\end{equation}  
Moreover, when ${\zeta_1}$ is larger in the sense of (\ref{eq_largeralphadefinition}), we have that with probability at least $1-\OO\left(n^{1-\delta({\zeta_1}-1)/2} \right),$ for some sufficiently small $\epsilon>0$ and some constant $C>0$
\begin{equation}\label{eq_thm35eqthree}
\left\|  \Nb-\widetilde{\Ab}_{2,s}  \right\| \leq  C \left( n^{-1/2+\epsilon}+n \exp(-\upsilon (\sigma_1^2/n)^{1-\delta}) \right). 
\end{equation} 
\item[(3).] When $\zeta_1 \geq \zeta_2 \geq 2,$ we have that 
\begin{equation}\label{eq_crossingtermproductconvergence}
\left\| \Nb- \widetilde{\Ab}_{1,c} \widetilde{\Ab}_{2,c}^\top\right \| \prec n^{-3/2}+n^{-\zeta_2/2}. 
\end{equation}
Moreover, if ${\zeta_2}$ is larger in the sense of (\ref{eq_largeralphadefinition}),  we have that with probability at least $1-\OO\left(n^{1-\delta(\min\{{\zeta_1}, {\zeta_2}\}-1)/2} \right),$ for some constant $C>0$, 
\begin{equation}\label{eqeqeqeqqqqrqrqer}
\left\| \Nb-\mathbf{I} \right\| \leq Cn \left( \exp(-\upsilon (\sigma_1^2/n)^{1-\delta})+ \exp(-\upsilon (\sigma_2^2/n)^{1-\delta})\right). 
\end{equation}
\end{enumerate}
\end{thm}

Theorem \ref{thm_secenario2moderately} shows that when $h_k=p_k$, where $k=1,2,$  and both SNRs are large, the NCCA matrix from the noisy dataset could be well approximated by that from the clean dataset of the common manifold. The main reason has been elaborated in \cite{DW2} when we have only one sensor. In the two sensors case, combining (\ref{eq_thm35eqone}) and (\ref{eq_thm35eqtwo}), we see that except the first $\mathsf{R}$ eigenvalues, the remaining eigenvalues are negligible and not informative. Moreover, (2) and (3) reveal important information about the bandwidth; that is, if the bandwidth choice is improper, like {$h_1=p_1$ and $h_2=p_2$,} the result could be misleading in general. For instance, when {$\zeta_1$ and $\zeta_2$} are large, ideally we should have a ``very clean'' dataset and we shall expect to obtain useful information about the signal. However, this result says that we cannot obtain any useful information from NCCA; particularly, see \eqref{eqeqeqeqqqqrqrqer}. This however is intuitively true, since when the bandwidth is too small, the relationship of two distinct points cannot be captured by the kernel; that is, when $i\neq j$, $\exp(-\|\xb_i-\xb_j\|^2/p_1)\approx 0$ with high probability (see proof below for a precise statement of this argument or \cite{DW2}). This problem can be fixed if we choose a proper bandwidth. In Section \ref{sec_adpativebandwidth}, we will state the corresponding results when the bandwidths are selected properly, in which case this counterintuitive result is eliminated.

\begin{rem}\label{rmk_generalspike}
In the above theorems, we focus on reporting the results for the case $d_1=d_2=1$ in (\ref{eq_spectraldecompositions1s2}). In this remark, we discuss how to generalize the results to the setting when $d_1 > 1$ or $d_2 > 1.$ First, when $0 \leq \sigma_{1i}^2, \sigma_{2j}^2<1, 1 \leq i \leq d_1,  1 \leq j \leq d_2,$ Theorem \ref{thm_senarioiionesubcritical} still holds after minor modification, for example, $\mathfrak{d}_k$ in (\ref{eq_defns1s2}) should be replaced by $\sum_{i=1}^{d_k} \mathfrak{d}_{ki}, \mathfrak{d}_{ki}:= \left \lceil  \frac{1}{1-{\zeta_{ki}}} \right \rceil+1$ and the error bound in (\ref{eq_finalresultthm31}) should be replaced by 
\begin{equation*}
\max\left\{ \max_i\{ n^{\frac{\zeta_{1i}-1}{2}}\}, \max_j\{n^{\mathsf{e}_{2j}}\}  \right\},
\end{equation*}
where $\mathsf{e}_{2j}$'s are defined similarly as in (\ref{Definition: zeta in Scenario II}).  Similar arguments apply for Theorem \ref{thm_secenario2slowly}.   Second, when $\zeta_{1i}, \zeta_{1j} \geq 1$, where $1 \leq i \leq d_1, 1 \leq j \leq d_2,$ Theorem \ref{thm_secenario2moderately} holds by setting $\zeta_k:=\max_j\{\zeta_{kj}\}, k=1,2.$ Finally, suppose that there exist some integers $r_k<d_k, k=1,2,$ such that 
\begin{equation*}
\zeta_{k1} \geq \zeta_{k2} \geq \cdots \geq \zeta_{k, r_k} \geq 1>\zeta_{k, r_k+1} \geq \cdots \zeta_{k, d_k} \geq 0. 
\end{equation*} 
Then we have that Theorem \ref{thm_secenario2moderately} still holds by setting $\zeta_k:=\zeta_{k1}, k=1,2,$ and the affinity and transition matrices in (\ref{eq_tildeaxs2}) should be defined using the signal part with large SNRs. For example, $\Wb_{1,s}$ should be defined via 
\begin{equation*}
\Wb_{1,s}(i,j)=\exp\left(-\upsilon \frac{\|\widetilde{\ub}_{ix}-\widetilde{\ub}_{jx} \|_2^2}{h_1} \right), \ \widetilde{\ub}_{ix}=(\ub_{ix}(1), \cdots, \ub_{ix}(r_1), 0, \cdots, 0).
\end{equation*}
The detailed statements and proofs are similar to the setting $d_1=d_2=1$ except for extra notational complicatedness. Since this is not the main focus of the current paper, we omit details here. 
\end{rem}

\begin{rem}\label{rmk_admatrix}
Throughout the paper, we focus on reporting the results of the NCCA matrix. However, our results can also be applied to the AD matrix with a minor modification based on their definitions in (\ref{eq_nccaadmatrixdefinition}). Specifically, Theorem \ref{thm_senarioiionesubcritical} holds for $n^2  \Ab$, Theorem \ref{thm_secenario2slowly} holds for $n \Ab$ and Theorem \ref{thm_secenario2moderately} holds for $\Ab$ by replacing $\widetilde{\Ab}_{1,s} \widetilde{\Ab}_{2,s}^\top$ with $\widetilde{\Ab}_{1,s} \widetilde{\Ab}_{2,s},$ $\widetilde{\Ab}_{1,c} \widetilde{\Ab}_{2,s}^\top$ with $\widetilde{\Ab}_{1,c} \widetilde{\Ab}_{2,s}$ and $\widetilde{\Ab}_{1,c} \widetilde{\Ab}_{2,c}^\top$ with $\widetilde{\Ab}_{1,c} \widetilde{\Ab}_{2,c}$. Since the proof is similar, we omit details. 
\end{rem}

\section{Main results (II)--adaptive choice of bandwidth}\label{sec_adpativebandwidth}  

As discussed after Theorem \ref{thm_secenario2moderately}, when the SNRs are large, the bandwidth choice $h_k=p_k$ for $k=1,2$ is improper.  
One solution to this trouble has been discussed in \cite{DW2} when we have one sensor; that is, the bandwidth is decided by the percentile of all pairwise distances. It is thus natural to hypothesize that the same solution would hold for the kernel sensor fusion approach. 
As in Section \ref{sec_mainresults}, we focus on the case $d_1=d_2=1,$ and the discussion for the general setting is similar to that of Remark  \ref{rmk_generalspike}.
As before, we also assume that $\zeta_1 \geq \zeta_2.$ 
Also, we focus on reporting the results of the NCCA matrix. The discussion for the AD matrix is similar to that of Remark \ref{rmk_admatrix}.

We first recall the adaptive bandwidth selection approach \cite{DW2} motivated by the empirical approach commonly used in daily practice. 
Let {$\nu_{\text{dist},1}$ and $\nu_{\text{dist},2}$} be the empirical distributions of pairwise distances $\{\| \xb_i-\xb_j\|\}_{i \neq j}$ and $\{\| \yb_i-\yb_j \|\}_{i \neq j}$ respectively. Then we choose the bandwidths {$h_1>0$ and $h_2>0$} by 
\begin{equation}\label{eq_hxhyadaptivechoise}
\int_{0}^{h_1} \dd \nu_{\text{dist},1}=\omega_1 \quad\mbox{and}\quad  \int_{0}^{h_2} \dd \nu_{\text{dist},2}=\omega_2,
\end{equation}
where $0<\omega_1<1$ and $0<\omega_2<1$ are fixed constants chosen by the user. Define $\widetilde{\Ab}_{1,s}$ in the same way as that in (\ref{eq_transitionshifting}),  $\widetilde{\Nb}$ as that in (\ref{eq_defnwidetilde}), and $\Ab_{1,s}$ as that in (\ref{eq_transitionsignal}) using (\ref{eq_hxhyadaptivechoise}). Similarly, we can define the counterparts for the point cloud $\mathcal{Y}.$

Recall that $\Wb_{1,s}$ and $\Wb_{2,s}$ are the affinity matrices associated with $\{\ub_{ix}\}$ and $\{\ub_{iy}\}.$ With a little bit abuse of notation, for $k=1,2,$ we denote 
\begin{align}\label{eq_widetildewxss22}
\widetilde{\Wb}_{k,s}
=\exp\left(-\upsilon\frac{2 p_k}{h_k} \right)  \Wb_{k,s}+\left(1-\exp\left(-\upsilon\frac{2 p_k  }{h_k} \right)\right)\Ib, 
\end{align}
where $\Wb_{k,s}$ are constructed using the adaptively selected bandwidth $h_k$. Clearly, $\Wb_{k,s}$ and $\widetilde{\Wb}_{k,s}$ differ by an isotropic spectral shift, and when $\zeta_k>1$, asymptotically $\Wb_{k,s}$ and $\widetilde{\Wb}_{k,s}$ are the same. Note that compared to (\ref{eq_widetildewxss2}), the difference is that we use the modified bandwidth in \eqref{eq_widetildewxss22}. 
This difference is significant, particularly when $\zeta_k$ is large. Indeed, when $\zeta_k$ is large, $\Wb_{k,s}$ defined in \eqref{eq_widetildewxss2} is close to an identity matrix, while $\Wb_{k,s}$ defined in \eqref{eq_widetildewxss22} encodes information of the signal. Specifically, asymptotically we can show that $\Wb_{k,s}$ defined in \eqref{eq_widetildewxss22} converges to an integral operator defined on the manifold, whose spectral structure is commonly used in manifold learning society to study the signal. See \cite{DW2} for more discussion.
We then define 
\begin{equation}\label{eq_transitionshifting}
\widetilde{\Ab}_{k,s}=\widetilde{\Db}_{k,s}^{-1} \widetilde{\Wb}_{k,s} 
\end{equation}
and
\begin{equation}\label{eq_transitionsignal}
\Ab_{k,s}=\Db_{k,s}^{-1} \Wb_{k,s}, 
\end{equation}
where $k=1,2$.
Compared to (\ref{eq_transitionshifting}), (\ref{eq_transitionsignal}) does not contain the scaling and shift of the signal parts. 
Moreover, denote
\begin{equation}\label{eq_defnnewvarsigmatauk}
\varsigma_{k,h} \equiv \varsigma_{k,h}(\tau_k, h_k):=1-\frac{2\upsilon p_k}{h_k} \exp\left(-\upsilon \frac{\tau_{k} p_k}{h_k}\right)-\exp\left(-\upsilon \frac{\tau_k p_k}{h_k}\right), \ k=1,2. 
\end{equation}

\begin{thm}\label{thm_mainbandwidthtwo} Suppose Assumption \ref{assu_main} holds with the adaptively chosen bandwidths $h_1,h_2$ and $d_1=d_2=1.$ Recall (\ref{eq_defnshiftoperator}), (\ref{eq_mp}) and (\ref{eq_defnnewvarsigmatauk}). Then we have that: 
\begin{enumerate}
\item[(1).] $0 \leq \zeta_2 < 1$ (at least one sensor has a low SNR). When $0 \leq \zeta_1 < 1$ as well, Theorem \ref{thm_senarioiionesubcritical} holds under the assumption of (\ref{eq_s2noisemodel}) and (\ref{eq_ratioone}) by replacing $\nu_k,  k=1,2,$ with
\begin{equation*}
\widetilde{\nu}_k:=\mathrm{T}_{\varsigma_{k,h}} \nu_{c_k, \eta_k},  \ \ \mbox{where}\ \ \eta_k=\frac{2 p_k \upsilon \exp(-2 p_k \upsilon/h_k)}{h_k} 
\end{equation*}
and $\exp(4 \upsilon)$ with $\exp(4 \upsilon p_1 p_2/(h_1 h_2))$ in (\ref{eq_finalresultthm31}). When $\zeta_1 \geq 1$, Theorem \ref{thm_secenario2slowly} holds by replacing $\Nb$ by $\widetilde{\Nb}$ in (\ref{eq_firstpartthm33}), where
\begin{equation}\label{eq_defnwidetilde}
\widetilde{\Nb}:=
\begin{cases}
\exp\left( \upsilon\frac{2 p_2}{h_2}\right) \widetilde{\Ab}_{1,s} \left( \varsigma_{2,h} \Ib+\frac{2 \upsilon \exp(-\upsilon p_2 \tau_{2} /h_2)}{h_2} \Wb^\top \Wb \right), & 1 \leq \zeta_1 <2; \\
\exp\left( \upsilon\frac{2 p_2}{h_2}\right)  \widetilde{\Ab}_{1,c} \left( \varsigma_{2,h} \Ib+\frac{2 \upsilon \exp(-\upsilon p_2 \tau_{2} /h_2)}{h_2} \Wb^\top \Wb \right), & \zeta \geq 2.
\end{cases}
\end{equation}
and when $\zeta_1$ is large in the sense of (\ref{eq_largeralphadefinition}), (\ref{eq_secondpartthm33}) holds replacing $\nu_2$ with $\widetilde{\nu}_2$ and $\exp(2 \upsilon)$ with $\exp(2 p_2 \upsilon/h_2).$
\item[(2).]  $\zeta_2 \geq 1$  (both sensors have high SNRs). In this case, we have that 
\begin{equation}\label{eq_thm41eqone}
\left\| \Nb-\widetilde{\Ab}_{1,s} \widetilde{\Ab}_{2,s}^\top \right\| \prec n^{-1/2}.
\end{equation}
Moreover, for some constant $C>0$ and $i \geq C \log n,$ we have 
\begin{equation}\label{eq_thm41eqtwo}
\lambda_i(\widetilde{\Ab}_{1,s} \widetilde{\Ab}_{2,s}^\top) \prec n^{-1}.
\end{equation}
Finally, when $\zeta_2>1,$ we have that  
\begin{equation}\label{eq_thm41eqthree}
\left\| \Nb-\Ab_{1,s} \Ab_{2,s}^\top \right\| \prec n^{-1/2}+n^{1-\zeta_2}.
\end{equation}
\end{enumerate}
\end{thm}

Theorem \ref{thm_mainbandwidthtwo} (1) states that if both sensors have low SNRs, the NCCA matrix has a similar spectral behavior as that in Theorem \ref{thm_senarioiionesubcritical}; that is, when the SNRs are small, due to the noise impact, even if there exists signal, we may not obtain useful result. The reason is that we still have $h_k \asymp  p_k$ for $k=1,2$  with high probability (see (\ref{eq_asymptoticrelation})), so the bandwidth choice does not influence the conclusion. Especially, most of the eigenvalues of $\Ab$ are governed by the free multiplication convolutions of two MP type laws, which are essentially the limiting empirical spectral distributions of Gram matrices only containing white noise.

On the other hand, when the signals are stronger; that is, $\zeta_1, \zeta_2 \geq 1,$ we are able to approximate the associated clean NCCA matrix of the underlying clean common component, as is detailed in Theorem \ref{thm_mainbandwidthtwo} (2).
This result can be interpreted as that NCCA is robust to the noise. Especially, when $\zeta_1, \zeta_2>1,$ we see that {$\Ab_{1,s}$ and ${\Ab}_{2,s}$} come from the clean dataset directly. Finally, we point out that compared to (\ref{eq_defnRindex}), except the top $C\log n$ eigenvalues for some constant $C>0$, the remaining eigenvalues are not information. When $\zeta_1, \zeta_2 \geq 2,$ the NCCA matrix is always informative compared to (2) and (3) of Theorem \ref{thm_secenario2moderately}. 
As a result, when combined with the existing theory about AD \cite{LEDERMAN2018509,TALMON2019848}, the first few eigenpairs of NCCA and AD capture the geometry of the common manifold under the manifold setup.

Theorem \ref{thm_mainbandwidthtwo} (1) also describes the behavior of NCCA when one sensor has a high SNR while the other one has a low SNR, which is the most interesting and counterintuitive case. In this case, even if the bandwidths of both sensors are generated according to (\ref{eq_hxhyadaptivechoise}), the NCCA matrix still encodes limited information about the signal, like that stated in Theorem \ref{thm_secenario2slowly}.
Indeed, the NCCA matrix is close to a product matrix which is a mixture of signal and noise, shown in \eqref{eq_defnwidetilde}. While $\widetilde{\Ab}_{1,s}$ contains information about the signal, it is contaminated by $\varsigma_{2,h} \Ib+\frac{2 \upsilon \exp(-\upsilon p_2 \tau_{2} /h_2)}{h_2} \Wb^\top \Wb$ via production, which comes from the noise dominant dataset collected from the other sensor. Since the spectral behavior of $\varsigma_{2,h} \Ib+\frac{2 \upsilon \exp(-\upsilon p_2 \tau_{2} /h_2)}{h_2} \Wb^\top \Wb$ follows the shifted and scaled MP law, overall we obtain limited information about the signal if we apply the kernel-based sensor fusion algorithm. In this case, it is better to simply consider the dataset with a high SNR.
Based on the above discussion and practical experience, we would like to mention a potential danger if we directly apply NCCA (or AD) without confirming the signal quality. This result warns us that if we directly apply AD without any sanity check, it may result in a misleading conclusion, or give us lower quality information. Therefore, before applying NCCA and AD, it is suggested to carry out the common practice by detecting the existence of signals in each of the sensors.

For the choices of the constants $\omega_1$ and $\omega_2,$
we comment that in practice, usually researchers choose $\omega_k=0.25$ or $0.5$ \cite{MR4010764}. In \cite{DW2}, we propose an algorithm to adaptively choose the values of them. The main idea behind is that the algorithm seeks for a bandwidth so that the affinity matrix has the most number of outlier eigenvalues. We refer the readers to \cite[Section 3.2]{DW2} for more details.

Last but not the least, we point out that our results can be potentially used to detect the common components. Usually, researchers count on the background knowledge to decide if common information exists. For example, it is not surprising that two electroencephalogram channels share the same brain activity. However, while physiologically the brain and heart share common information \cite{samuels2007brain}, it is less clear if the electroencephalogram and the electrocardiogram share anything in common, and what is the common information. Answering this complicated question may need a lot of scientific work, but the first step toward it is a powerful tool to confirm if two sensors share the same information, in addition to checking if the signal quality is sufficient. Since this is not the focus of the current paper, we will address this issue in the future work.

\section{Proof of main theorems}\label{sec_proofs}

We now provide proofs of the main theoretical results in Sections \ref{sec_mainresults} and \ref{sec_adpativebandwidth}.
We start from collecting some technical lemmas needed in the proof.

\subsection{Some technical lemmas}\label{sec_auxilemma}

The following lemma provides some deterministic inequalities for the products of matrices.
\begin{lem}\label{lem_hardmard}
(1). Suppose that $\Lb$ is a real symmetric matrix with nonnegative entries and $\Eb$ is another real symmetric matrix. Then we have that
\begin{equation*}
\sigma_1(\Lb \circ \Eb) \leq \max_{i,j} |\Eb(i,j)| \sigma_1(\Lb),
\end{equation*}
where $\Lb \circ \Eb$ is the Hadamard product and $\sigma_1(\Lb)$ stands for the largest singular value of $\Lb.$ \\
(2). Suppose $A$ and $B$ are two $n \times n$ positive definite matrices. Then for all $1 \leq k \leq n,$ we have that 
\begin{equation*}
\lambda_k(A) \lambda_n(B) \leq \lambda_k(AB) \leq \lambda_k(A) \lambda_1(B).
\end{equation*}
\end{lem}
\begin{proof}
For (1), see \cite[Lemma A.5]{NEKkernel}. (2) follows from Courant-Fischer-Weyl's min-max principle via
\begin{align*}
\lambda_k(AB)=\lambda_k(\sqrt{B}A\sqrt{B})
&=\min_{\substack{F\subset \mathbb{R}^n \\ \dim(F)=k}} \left( \max_{x\in F\backslash \{0\}} \frac{(\sqrt{B}A\sqrt{B}x,x)}{(x,x)}\right)\\
&=\min_{\substack{F\subset \mathbb{R}^n \\ \dim(F)=k}} \left( \max_{x\in F\backslash \{0\}} \frac{(A\sqrt{B}x,\sqrt{B}x)}{(\sqrt{B}x,\sqrt{B}x)}
\frac{(Bx,x)}{(x,x)}\right).
\end{align*}
\end{proof}

 The following Lemma \ref{lem_con} collects some concentration inequalities. 

\begin{lem}\label{lem_con} Suppose Assumption \ref{assu_main} holds with $d_1=d_2=1$.  Moreover, we assume that $0 \leq \zeta_1, \zeta_2<1 $ in (\ref{eq_variancesnrration1}). Then we have that 
\begin{equation}\label{concentration1}
\frac{1}{p_1} \left| \xb_i^\top \xb_j \right| \prec \frac{\sigma^2_1}{n}+\frac{1}{\sqrt{n}},  \  \frac{1}{p_2} \left| \yb_i^\top \yb_j \right| \prec \frac{\sigma^2_2}{n}+\frac{1}{\sqrt{n}},
\end{equation} 
and 
\begin{equation}\label{concentration2}
\left| \frac{1}{p_1} \| \xb_i \|_2^2-\left(1+\frac{\sigma^2_1}{p_1} \right) \right|\prec \frac{\sigma^2_1}{n}+\frac{1}{\sqrt{n}}, \  \left| \frac{1}{p_2} \| \yb_i \|_2^2-\left(1+\frac{\sigma^2_2}{p_2} \right) \right|\prec \frac{\sigma^2_2}{n}+\frac{1}{\sqrt{n}}.
\end{equation}
\end{lem}
\begin{proof}
See Lemma A.2 of \cite{DW2}. 
\end{proof}

In the following lemma, we prove some results regarding the concentration of the affinity matrices when $0 \leq \zeta_1<1$ and $0 \leq \zeta_2<1$. 

\begin{lem}\label{lem_conmatrix}
 Follow the notations (\ref{eq_sho})--(\ref{eq_sh2}).
Recall (\ref{eq_defns1s2}) for $d_1=d_2=1$. For $f(x)=\exp(-\upsilon x),$ we denote $\rSh_\sd$ and $\widetilde{\rSh}_d$ such that 
\begin{equation}\label{eq_defnshd}
\rSh_\sd(i,j):=\sum_{k=3}^{\mathfrak{d}_1-1} \frac{f^{(k)}(\tau_1) \Lb_x(i,j)^k}{k!}, \ \widetilde{\rSh}_\sd(i,j):=\sum_{k=3}^{\mathfrak{d}_2-1} \frac{f^{(k)}(\tau_2) \Lb_y(i,j)^k}{k!},
\end{equation}
where $\Lb_x=\Ob_x-\Pb_x,$ and $\Ob_x$ and $\Pb_x $ are defined as 
\begin{equation*}
\Ob_x(i,j)=(1-\delta_{ij})(\phi_i+\phi_j), \ \Pb_x(i,j)=(1-\delta_{ij}) \frac{\xb_i^\top \xb_j}{p_1}. 
\end{equation*}
Denote 
\begin{equation}\label{eq_kdtau}
\Kb_1=
\begin{cases}
-2f'(\tau_1) p_1^{-1} \Xb^\top \Xb+\varsigma_1 \mathbf{I}_n+\mathrm{Sh}_{10}(\tau)+\mathrm{Sh}_{11}(\tau)+\mathrm{Sh}_{12}(\tau), & 0 \leq \alpha_1<0.5 \\
-2f'(\tau_1) p_1^{-1} \Xb^\top \Xb+\varsigma_1 \mathbf{I}_n+\mathrm{Sh}_{10}(\tau)+\mathrm{Sh}_{11}(\tau)+\mathrm{Sh}_{12}(\tau)+\rSh_\sd, & 0.5 \leq \alpha_1<1.
\end{cases}
\end{equation} 
Recall (\ref{Definition: zeta in Scenario II}).  Then, when $0 \leq \zeta_1<1$ and $h_1=p_1,$ we have  
\begin{equation}\label{kernelmatrixresults}
\Wb_1=\Kb_1+\OO_{\prec}(n^{\mathsf{e}_1}+n^{-1/2}). 
\end{equation}
Moreover, we have   
\begin{equation}\label{normalkernelmatrixresults}
\Ab_1=\frac{1}{nf(\tau_1)} \Kb_1+\OO_{\prec}(n^{\mathsf{e}_1}+n^{-1/2}).   
\end{equation}
Finally, for $\mathsf{s}_1$ in (\ref{eq_defns1s2}), we have 
\begin{equation}\label{eq_rank}
\operatorname{rank}(\rSh_\sd) \leq \mathsf{s}_1. 
\end{equation}
Similar results hold for $\Ab_2$ and $\Wb_2$ using $\widetilde{\phi}_k, 1 \leq k \leq n,$ $\widetilde{\rSh}_\sd$ and $\rSh_{2i}, i=0,1,2.$ 
\end{lem}
\begin{proof}
First, (\ref{kernelmatrixresults}) has been proved in \cite{DW2} using the entry-wise Taylor expansion and the Gershgorin circle theorem; see the proof of Theorems 2.3 and 2.5 of \cite{DW2}. Second, (\ref{eq_rank}) has been proved in the proof of Theorem 2.5 of \cite{DW2}.  Third, we prove (\ref{normalkernelmatrixresults}). By Lemma \ref{lem_con} and a discussion similar to \cite[Lemma IV.5]{DW1}, when $0 \leq \zeta_1<1$ and $0 \leq \zeta_2<1$
\begin{equation}\label{eq_concentrationdiagonal}
\left\|(n\Db_1)^{-1}-\frac{1}{f(\tau_1)} \Ib \right\| \prec n^{\mathsf{e}_1}+n^{-1/2}, \ \left\|(n\Db_2)^{-1}-\frac{1}{f(\tau_2)} \Ib \right\| \prec n^{\mathsf{e}_2}+n^{-1/2}.
\end{equation} 
Consequently, 
\begin{align*}
\left\| n\Ab_1-  \frac{1}{f(\tau_1)} \Kb_1 \right\| & \leq  \left\| (n\Db_1)^{-1} \Wb_1-  \frac{1}{f(\tau_1)} \Wb_1   \right\|+\frac{1}{f(\tau_1)}\left\| \Wb_1-\Kb_1 \right\| \\
& \prec  (n^{\mathsf{e}_1}+n^{-1/2})( \| \Wb_1 \|+1),
\end{align*}
where we used the fact that $\tau_1<\infty.$ Since $\max_{i,j}|\Wb_1(i,j)| \prec 1,$ by the Gershgorin circle theorem, we conclude that $\| \Wb_1 \| \prec n.$ This concludes our proof. 
\end{proof}

In the following lemma, we collect the results regarding the affinity matrices when $\zeta_1 \geq 1$ and $\zeta_2 \geq 1. $ Recall $\widetilde{\Ab}_{1,s}$ defined via (\ref{eq_tildeaxs2}). 
\begin{lem}\label{lem_2ndpaperresults} Suppose  Assumption \ref{assu_main} holds with $d_1=d_2=1$, $\zeta_1, \zeta_2 \geq 1$. For some constant $C>0,$ denote
\begin{equation}\label{eq_defnt1}
\mathsf{T}_1:=
\begin{cases}
C \log n, & \zeta_1=1; \\
C n^{\zeta_1-1}, & 1 <\zeta_2<2. 
\end{cases}
\end{equation}
Then we have:  \\
(1). When $h_1=p_1,$ if $1 \leq \zeta_1<2,$
\begin{equation}\label{eq_onematrixbound}
\left\| \Ab_1-\widetilde{\Ab}_{1,s} \right\| \prec n^{-1/2}. 
\end{equation}
Moreover, moreover, we have that for $i>\mathsf{T}_1$ in (\ref{eq_defnt1}), 
\begin{equation}\label{eq_onematrixboundoneone}
\lambda_{i}(\widetilde{\Ab}_{1,s}) \prec n^{(\zeta_1-3)/2}.
\end{equation}
On the other hand, when $\zeta_1 \geq 2,$ we have that
\begin{equation*}
\| \Ab_1-\widetilde{\Ab}_{1,c} \| \prec n^{-\zeta_1/2}+n^{-3/2}. 
\end{equation*}
Finally, when $\zeta_1$ is the larger in the sense that (\ref{eq_largeralphadefinition}) holds, we have that with probability at least $1-\OO(n^{1-\delta(\zeta_1-1)/2}),$ for some constant $C>0,$ 
\begin{equation}\label{eq_indentitybound}
\left\| \Ab_1-\Ib \right\| \leq C n \exp\left(-\upsilon n^{(\zeta_1-1)(1-\delta)} \right). 
\end{equation}
(2). When $h_1$ is chosen according to \eqref{eq_hxhyadaptivechoise},  we have that
\begin{equation}\label{eqa}
\left\| \Ab_1-\widetilde{\Ab}_{1,s} \right\| \prec n^{-1/2}.
\end{equation} 
Moreover, we have that for $i>C \log n$ for some constant $C>0$,
\begin{equation}\label{eqb}
\lambda_{i}(\widetilde{\Ab}_{1,s}) \prec n^{-1}.
\end{equation}
Recall (\ref{eq_transitionsignal}). Finally, when $\zeta_1>1,$ we have that
\begin{equation}\label{eqc}
\left\| \Ab_1-\Ab_{1,s} \right\| \prec n^{-1/2}+n^{1-\zeta_1}.
\end{equation} 
Similar results hold for $\Ab_2. $
\end{lem}
\begin{proof}
See Corollary 2.11 and Theorem 3.1 of \cite{DW2}. 
\end{proof}

Finally, we record the results for the rigidity of eigenvalues of  non-spiked Gram matrix. Denote the non-spiked  Gram matrix as $\Sb$, where
\begin{equation*}
\Sb:=\frac{1}{p_1} \Zb^\top \Zb,
\end{equation*}
and its eigenvalues as $\lambda_1 \geq \cdots \geq \lambda_{n}.$ Recall (\ref{eq_mp}) and (\ref{eq_defnclassiciallocaltion}). 

\begin{lem}\label{lem_rigiditysamplecovariance}
Suppose $\{\zb_i\}$ are sub-Gaussian random vectors satisfying (\ref{eq_s2noisemodel11}), (\ref{assumptuon_ratio}) and (\ref{eq_ratioone}). Then we have
\begin{equation*}
\sup_j \left| \lambda_j-\gamma_{\nu_{c_1,1}}(j) \right| \prec n^{-2/3} \widetilde{j}^{-1/3}, \ \widetilde{j}:=\min\left\{p_1 \wedge n+1-j,j\right\}.
\end{equation*}

\end{lem}

\begin{proof}
See \cite[Theorem 3.3]{PY}.  
\end{proof}

\subsection{Proof of Theorem \ref{thm_senarioiionesubcritical}}

We need the following notations. Denote $\Phi_1=(\phi_{1,1},\ldots, \phi_{1,n})$, where $\phi_{1,i}=\frac{1}{p_1}\|\xb_i\|_2^2-(1+\sigma_1^2/p_1)$, $i=1,2,\cdots, n$. Similarly, we define $\Phi_2=({\phi}_{2,1},\ldots, {\phi}_{2,n})$ with ${\phi}_{2,i}=\frac{1}{p_2}\|\yb_i\|_2^2-(1+\sigma_2^2/p_2)$, $i=1,2,\cdots, n.$  For $k=1,2,$ denote 
\begin{align}
& \mathrm{Sh}_{k0}(\tau_k):= f(\tau_k) \mathbf{1} \mathbf{1}^\top,\label{eq_sho} \\
& \mathrm{Sh}_{k1}(\tau_k):=f'(\tau_k)[\mathbf{1} \Phi_k^\top+\Phi_k \mathbf{1}^\top ],  \label{eq_sh1} \\
& \mathrm{Sh}_{k2}(\tau_k):= \frac{f^{(2)}(\tau_k)}{2}\left[ \mathbf{1}( \Phi_k \circ \Phi_k)^\top+(\Phi_k \circ \Phi_k) \mathbf{1}^\top+2 \Phi_k \Phi_k^\top +4\frac{(\sigma_k^2+1)^2+p_k}{p_k^2} \mathbf{1} \mathbf{1}^\top \right]. \label{eq_sh2}
\end{align}

\begin{proof}[\bf Proof] {\bf Case (1). \underline{$0 \leq \zeta_2 \leq \zeta_1<0.5$}}. By (\ref{eq_concentrationdiagonal}), we conclude that  
\begin{equation}\label{eq_diagonalconvergence}
\left\|n\Db_1^{-1}-\exp( 2\upsilon) \Ib \right\|=\OO_{\prec}(n^{-1/2}), \ \left\|n\Db_2^{-1}-\exp( 2\upsilon) \Ib \right\|=\OO_{\prec}(n^{-1/2}). 
\end{equation} 
Therefore, it suffices to consider $\Wb_1 \Wb_2.$ 
To ease the heavy notation, we denote 
\begin{equation}\label{eq_defnelemntaryquantityuselater}
\ell_{kt}:=(-\upsilon)^t \exp(-\upsilon \tau_k), \ \rSh_k=\sum_{j=0}^2 \rSh_{kj},
\end{equation}
where $k=1,2$ and $t \in \mathbb{N}$.
With the above notations,  by Lemma \ref{lem_conmatrix}, we have that 
\begin{align}\label{eq_spectraldecomposition}
&(\Wb_1-\rSh_1 )( \Wb_2-\rSh_2)\\
=&\, \left(-2 \frac{\ell_{11}}{p_1} \Xb^\top \Xb+\varsigma_1 \Ib+ \OO_{\prec}(n^{-1/2})\right) \left( -2 \frac{\ell_{21}}{p_2} \Yb^\top \Yb+\varsigma_2 \Ib+ \OO_{\prec}(n^{-1/2}) \right).\nonumber
\end{align}
Denote 
\begin{equation}\label{eq_defnpb1pb2}
\Pb_1:=-2 \frac{\ell_{11}}{p_1} \Xb^\top \Xb+\varsigma_1 \Ib\ \ \mbox{and}\ \ \Pb_2:= -2 \frac{\ell_{21}}{p_2} \Yb^\top \Yb+\varsigma_2 \Ib.
\end{equation}
Since $d_1=d_2=1$ and $\ub_{ix}$ and $\ub_{iy}$ contain samples from the common manifold,   we can set
\begin{equation*}
\bm{u}=(u_1, \cdots, u_n)^\top \in \mathbb{R}^n.
\end{equation*}
Moreover, denote 
\begin{equation*}
\bm{z}=(\zb_{i1}, \cdots, \zb_{in} )^\top  \ \ \mbox{and}\ \  \bm{w}=(\wb_{i1}, \cdots, \wb_{in})^\top \in \mathbb{R}^n. 
\end{equation*} 
Withe above notations, denote 
\begin{equation}\label{eq_defndelta1delta2}
\Delta_1:=-2\frac{\ell_{11}}{p_1} \bm{u} \bm{u}^\top, \ \Delta_2:=-2\frac{\ell_{21}}{p_2} \bm{u} \bm{u}^\top,
\end{equation}
\begin{equation}\label{eq_defndelta1delta21}
\Upsilon_1:=-2\frac{\ell_{11}}{p_1}(\bm{u} \bm{z}^\top+\bm{z} \bm{u}^\top), \ \Upsilon_2:=-2\frac{\ell_{21}}{p_2} (\bm{u} \bm{w}^\top+\bm{w} \bm{u}^\top),
\end{equation}
\begin{equation*}
\Tb_1=-2\frac{\ell_{11}}{p_1} \Zb^\top  \Zb+\varsigma_1 \Ib \ \ \mbox{and} \ \ \Tb_2=-2\frac{\ell_{21}}{p_2} \Wb^\top  \Wb+\varsigma_2 \Ib. 
\end{equation*}
Note that
\begin{equation}\label{eq_decompoistionp1p2}
\Pb_k=\Tb_k+\Delta_k+\Upsilon_k,  \ k=1,2. 
\end{equation}
Moreover, $\Delta_k$ are rank-one matrices and by (\ref{concentration1}),
\begin{equation}\label{eq_upsiolonbound}
\Delta_k=\OO_{\prec} \left(n^{\zeta_k}\right), \ \Upsilon_k=\OO_{\prec}\left( n^{\frac{\zeta_k-1}{2}} \right).
\end{equation}
In light of (\ref{eq_decompoistionp1p2}), we can write
\begin{align}\label{eq_p1p2spikedecomposition}
\Pb_1 \Pb_2=\prod_{k=1}^2 (\Tb_k+\Delta_k+\Upsilon_k).
\end{align}    
 We can further write 
\begin{align}\label{eq_defnt1t2}
\Pb_1 \Pb_2=\Tb_1 \Tb_2+\Rb_1+\Rb_2,
\end{align}
where 
\begin{equation*}
\Rb_1:=\Tb_1 \Delta_2+\Delta_1 \Pb_2\ \ \mbox{and}\ \ 
\Rb_2:=\Tb_1 \Upsilon_2+\Upsilon_1 \Pb_2. 
\end{equation*}
On one hand, it is easy to see that $\operatorname{rank}(\Rb_1) \leq 2.$ On the other hand, by (\ref{eq_upsiolonbound}), (\ref{concentration2}) and Lemma \ref{lem_rigiditysamplecovariance}, using the assumption that $\zeta_1 \geq \zeta_2,$ we obtain that
\begin{equation}\label{eq_r2bound}
\Rb_2=\OO_{\prec}(n^{\frac{\zeta_1-1}{2}}). 
\end{equation}
Denote the spectral decompositions of $\Tb_1$ and $\Tb_2$ as 
\begin{equation}\label{eq_defnp1p2}
\Tb_1=\Ub_1 \bm{\Lambda}_1 \Ub_1^\top, \ \Tb_2=\Ub_2 \bm{\Lambda}_2 \Ub_2^\top.
\end{equation}  
Let $\{a_k\}$ and $\{b_k\}$  be the quantiles of $\nu_1$ and $\nu_2,$ respectively  as constructed via (\ref{eq_rigiditydefinition}). Let $\{\lambda_i\}$ be the eigenvalues of $\Tb_1$. For some small $\epsilon>0,$ we denote an event
\begin{equation}\label{eq_highprobabilitysets}
 \Xi:=
 \left\{\sup_{i \geq 1}|\lambda_i-a_i| \leq \widetilde{i}^{-1/3} n^{-2/3+\epsilon} \right\} , \ \widetilde{i}:=\min\{(p_1-1) \wedge n+1-j,j\}.  
\end{equation}
Since $\Wb$ is a Gaussian random matrix, we have that $\Ub_2$ is a Haar orthogonal random matrix. Since $\Zb$ and $\Wb$ are independent, we have that $\Ub:=\Ub_2^\top \Ub_1$ is also a Haar orthogonal random matrix when $\Ub_1$ is fixed. Since 
 Lemma \ref{lem_rigiditysamplecovariance} implies that $\Xi$ is a high probability event, in what follows, we focus our discussion on the high probability event $\Xi$ and $\Ub_1$ is a deterministic orthonormal matrix.

 On one hand, $\Tb_1 \Tb_2$ have the same eigenvalues as $\bm{\Lambda}_2 \Ub \bm{\Lambda}_1 \Ub^\top$ by construction. On the other hand, by Lemma \ref{lem_rigiditysamplecovariance}, we have that for $\Hb:=\Sigma_2 \Ub \Sigma_1 \Ub^\top$ 
\begin{equation*}
 \left\| \bm{\Lambda}_2 \Ub \bm{\Lambda}_1 \Ub^\top-\Hb \right\| \prec n^{-2/3},  
\end{equation*}     
where $\Sigma_1$ and $\Sigma_2$ are diagonal matrices  containing $\{a_k\}$ and $\{b_k\},$ respectively. Note that the rigidity of the eigenvalues of $\Hb$ has been studied in \cite{DJ} under the Gaussian assumption and summarized in Lemma \ref{lem_multiplication}. Together with Lemma \ref{lem_multiplication}, we conclude that for $i \geq 1$
\begin{equation}\label{eq_edgeedge}
\left| \lambda_i(\Tb_1 \Tb_2)- \gamma_{\nu_1 \boxtimes \nu_2}(i) \right| \prec n^{-2/3}.
\end{equation}
Note that
\begin{align}\label{eq_n2ndecomposition}
n^2 \Nb=&n^2 \Db_1^{-1} (\Wb_1-\rSh_1) (\Wb_2-\rSh_2) \Db_2^{-1} \\
&+ n^2 \Db_1^{-1} \left(  \Wb_1 \rSh_2+\rSh_1 \Wb_2-\rSh_1 \rSh_2  \right)  \Db_2^{-1}. \nonumber
\end{align}
We then analyze the rank of $\Wb_1 \rSh_2+\rSh_1 \Wb_2 +\rSh_1 \rSh_2.$ Recall that for any compatible  matrices $A$ and $B,$ we have that
\begin{equation*}
\operatorname{rank}(AB) \leq \min\{\operatorname{rank}(A), \operatorname{rank}(B)\}. 
\end{equation*}
Since $\operatorname{rank}(\rSh_1) \leq 3$ and $\operatorname{rank}(\rSh_2) \leq 3,$ we conclude that 
\begin{equation*}
\operatorname{rank}(\Wb_1 \rSh_2+ \rSh_1 \Wb_2+\rSh_1 \rSh_2) \leq 6. 
\end{equation*}
Consequently, we have that
\begin{equation}\label{eq_rank6matrix}
\operatorname{rank}\left( \Rb\right) \leq 6, \ \mbox{where}\ \ \Rb:=n^2 \Db_1^{-1} \left(  \Wb_1 \rSh_2+\rSh_1 \Wb_2-\rSh_1 \rSh_2  \right)  \Db_2^{-1}.  
\end{equation} 
By (\ref{eq_p1p2spikedecomposition}), (\ref{eq_defnt1t2}), (\ref{eq_n2ndecomposition}) and (\ref{eq_rank6matrix}), utilizing  (\ref{eq_diagonalconvergence}), we obtain that 
\begin{equation}\label{eq_expanexpan}
n^2 \Nb=\frac{1}{\exp(-4 \upsilon)} \Tb_1 \Tb_2+ n^2 \Db_1^{-1} \Rb_1 \Db_2^{-1}+\Rb+\OO_{\prec}(n^{\frac{\zeta_1-1}{2}}),
\end{equation}
where we used $\| \Tb_1 \| \prec 1, \|  \Tb_2 \| \prec 1.$
Since $\operatorname{rank}(\Rb) \leq 6$, $\operatorname{rank}(\Rb_1) \leq 2,$ by (\ref{eq_edgeedge}),  we have finished our proof for case (1). 

\vspace{6pt}

\noindent {\bf Case (2). \underline{$0 \leq \zeta_2<0.5 \leq \zeta_1<1$}}. Recall (\ref{eq_decompoistionp1p2}). In this case, according to Lemma \ref{lem_conmatrix}, we require a  high order expansion up to the degree of $\mathfrak{d}_1$ in (\ref{eq_defns1s2}) for $\Wb_1.$ Recall (\ref{Definition: zeta in Scenario II}) and (\ref{eq_p1p2spikedecomposition}). By Lemma \ref{lem_conmatrix},  we have that 
\begin{equation}\label{eq_decompositionwxslowly}
\Wb_1-\rSh_\sd-\rSh_1=\Pb_1+\OO_{\prec} \left(n^{\mathsf{e}_1}\right),
\end{equation} 
where $\rSh_1$ is defined in (\ref{eq_defnelemntaryquantityuselater}) and $\rSh_\sd$ is defined in (\ref{eq_defnshd}) below satisfying $\operatorname{rank}(\rSh_\sd) \leq \mathsf{s}_1, \ 4 \leq \mathsf{s}_1 \leq C4^{\mathfrak{d}_1}.$ Using a decomposition similar to (\ref{eq_n2ndecomposition}) with (\ref{eq_decompositionwxslowly}), by (\ref{eq_upsiolonbound}) and the assumption $\zeta_1 \geq \zeta_2$, we obtain that 
\begin{align}
n^2 \Nb =&\, n^2 \Db_1^{-1} \left(\Tb_{1}+\Delta_{1}+\Upsilon_1+\OO_{\prec}(n^{\mathsf{e}_1}) \right) (\Tb_2+\Delta_2+\Upsilon_2+\OO_{\prec}(n^{-1/2})) \Db_2^{-1} \nonumber \\
& + n^2 \Db_1^{-1} \left(  (\Wb_1-\rSh_1-\rSh_\sd)\rSh_2+(\rSh_1+\rSh_\sd) \Wb_2 \right)  \Db_2^{-1} \nonumber \\
=&\, n^2 \Db_1^{-1} \left(\Tb_{1}+\OO_{\prec}(n^{\mathsf{e}_1}) \right) (\Tb_2+\OO_{\prec}(n^{-1/2})) \Db_2^{-1}\label{eq_slowlydivergentdecompose}\\
&+ n^2 \Db_1^{-1} \Delta_{1} (\Pb_2+\OO_{\prec}(n^{-1/2})) \Db_2^{-1} \nonumber \\
&+n^2\Db_1^{-1} \Tb_1 (\Delta_2+\OO_{\prec}(n^{-1/2}))\Db_2^{-1}  \nonumber \\
&+  n^2 \Db_1^{-1} \left(  (\Wb_1-\rSh_1-\rSh_\sd)\rSh_2+(\rSh_1+\rSh_\sd) \Wb_2 \right)  \Db_2^{-1}+\OO_{\prec}(n^{\frac{\zeta_1-1}{2}}).\nonumber
\end{align}
 It is easy to see that the rank of the second to the fourth terms  of (\ref{eq_slowlydivergentdecompose}) is bounded by $\mathsf{s}_1+8.$ On other hand, by (\ref{eq_concentrationdiagonal}), the first inequality of (\ref{eq_diagonalconvergence}) should be replaced by
\begin{equation}\label{eq_slowlydivergentdiagonalconcetrantion}
\left\|n(\Db_1)^{-1}-\exp( 2\upsilon) \Ib \right\|=\OO_{\prec}(n^{\zeta_1-1}).
\end{equation}
The rest of the discussion follows from the case (1). This completes our proof for case (2). 

\vspace{6pt}

\noindent {\bf Case (3). \underline{$0.5 \leq \zeta_2 \leq \zeta_1<1$}}. The discussion is similar to case (2) except that we also need to conduct a high order expansion for $\Wb_2.$ Similar to (\ref{eq_decompositionwxslowly}), by Lemma \ref{lem_conmatrix}, we have that 
\begin{equation}\label{wyslowlyccc}
\Wb_2-\widetilde{\rSh}_\sd-\rSh_2=\Pb_2+\OO_{\prec} \left(n^{\mathsf{e}_2}\right).
\end{equation}
By decomposition similar to (\ref{eq_slowlydivergentdecompose}), with (\ref{wyslowlyccc}), by (\ref{eq_upsiolonbound}), we have that 
\begin{align*}
n^2 \Nb=&\,n^2 \Db_1^{-1} \left( \Tb_{1}+\OO_{\prec}(n^{\mathsf{e}_1}) \right) \left(\Tb_{2}+\OO_{\prec}(n^{\mathsf{e}_2}) \right) \Db_2^{-1} \\
&+n^2 \Db_1^{-1} \left( (\Delta_{1}+\rSh_\sd+\rSh_1) \Wb_2+(\Wb_1-\Pb_1-\rSh_\sd- \rSh_1) (\widetilde{\rSh}_\sd+\Delta_2+\rSh_2) \right) \Db_2^{-1}\\
&+\OO_{\prec}(n^{\frac{\zeta_1-1}{2}}). 
\end{align*}  
On one hand,  the rank of the second term of right-hand side of the above equation can be bounded by $\mathsf{s}_1+\mathsf{s}_2+8.$ On the other hand, the first term can be again analyzed in the same way as  heading from (\ref{eq_highprobabilitysets}) to (\ref{eq_edgeedge}) using Lemma \ref{lem_multiplication}. Finally, by (\ref{eq_concentrationdiagonal}), similar to (\ref{eq_slowlydivergentdiagonalconcetrantion}), we have that
\begin{equation}\label{eq_slowlydivergentdiagonalconcetrantion1}
\left\|n(\Db_2)^{-1}-\exp( 2\upsilon) \Ib \right\|=\OO_{\prec}(n^{\zeta_2-1}).
\end{equation}
The rest of the proof follows from the discussion of case (1). This completes the proof of Case (3) using the fact $\zeta_2 \leq \zeta_1.$
\end{proof}

\subsection{Proof of Theorem \ref{thm_secenario2slowly}}
We now prove Theorem \ref{thm_secenario2slowly} when $0 \leq \zeta_2<1 \leq \zeta_1<\infty$.

{\bf Case (1). \underline{$0 \leq \zeta_2<0.5$}.}  Decompose $\Nb$ by
\begin{align}\label{eq_decompositionNmoderate}
n \Nb=\Ab_1 n(\Wb_2-\rSh_2) \Db_2^{-1}+n \Ab_1\rSh_2 \Db_2^{-1}.
\end{align}
First, we have that $\operatorname{rank}(n \Ab_1\rSh_2 \Db_2^{-1})\leq 3. $ Moreover, using the decomposition (\ref{eq_decompoistionp1p2}), similar to (\ref{eq_expanexpan}), by (\ref{eq_upsiolonbound}), we can further write that 
\begin{equation}\label{eq_decompositionNmoderate1}
n \Nb=n\Ab_1 \Tb_1 \Db_2^{-1}+n \Ab_1\rSh_2 \Db_2^{-1}+n \Ab_1 \Delta_2 \Db_2^{-1}+\OO_{\prec}(n^{\frac{\zeta_2-1}{2}}).
\end{equation}
Second, by Lemma \ref{lem_2ndpaperresults}, we have that 
\begin{equation}\label{eq_axaxtildeconverge}
\left\| \Ab_1-\widetilde{\Ab}_{1,s} \right\| \prec n^{-1/2}, \ \  \left\| \Ab_1-\widetilde{\Ab}_{1,c} \right\| \prec n^{-\alpha/2}+n^{-3/2}. 
\end{equation}
Together with (\ref{eq_diagonalconvergence}) and the fact that $\|\Tb_1\|=\OO_{\prec}(1)$, using the definition (\ref{eq_defnM1M2}),  we have that 
\begin{equation}\label{eq_moderatelyly}
\left\| \Ab_1 n \Tb_1 \Db_2^{-1}-\widetilde{\Nb} \right\| \prec n^{-1/2}. 
\end{equation}
We can therefore conclude the proof  using (\ref{eq_decompositionNmoderate1}).  

Next, when $\zeta_1$ is larger in the sense of (\ref{eq_largeralphadefinition}), by Lemma \ref{lem_2ndpaperresults}, we find that with probability at least $1-\OO\left(n^{1-\delta(\zeta_1-1)/2} \right)$, for some constant $C>0,$
\begin{equation}\label{eq_convergenidentity}
\left\|\Ab_1-\Ib \right\| \leq C n \exp\left(-\upsilon (\sigma_1^2/n )^{1-\delta} \right).
\end{equation}  
Consequently, we have 
\begin{equation}\label{eq_fastonenoise}
\left\| n\Nb-n \Ab_2 \right\| \leq \left\| \Ab_1-\Ib \right\| \left\| n\Ab_2 \right\| \leq  n^2 \exp\left(-\upsilon (\sigma_1^2/n )^{1-\delta} \right),
\end{equation}
where in the second inequality we use the fact that $\|\Ab_2  \| \prec 1$ since $\lambda_1(\Ab_2)=1$. By a result analogous to (\ref{normalkernelmatrixresults}) for $\Ab_2$, we have that for $i>3,$ 
\begin{equation}\label{eq_fastdivergentregime}
\left|\lambda_i(n \Ab_2)-\exp(2 \upsilon) \gamma_{\nu_2}(i) \right| \prec n^{-1/2}. 
\end{equation} 
Together with (\ref{eq_fastonenoise}), we conclude our proof. 
 
\vspace{6pt}

\noindent {\bf Case (2). \underline{$0.5 \leq \zeta_2<1$}.} The discussion is similar to case (1) except that we need to conduct a high order expansion for $\Ab_2.$ Note that 
\begin{align*}
n \Nb=\Ab_1 n(\Wb_2-\rSh_2-\widetilde{\rSh}_\sd) \Db_2^{-1}+n \Ab_1 (\rSh_2+\widetilde{\rSh}_\sd) \Db_2^{-1}.
\end{align*}
By (\ref{eq_upsiolonbound}), (\ref{wyslowlyccc}) and (\ref{eq_slowlydivergentdiagonalconcetrantion1}), we have that 
\begin{align*}
n \Nb=\exp(2 \upsilon )\Ab_1 \Tb_2+n \Ab_1 \Delta_1 \Db_2^{-1}+\OO_{\prec}(n^{\frac{\zeta_2-1}{2}}+n^{\mathsf{e}_2})+n \Ab_1 (\rSh_2+\widetilde{\rSh}_\sd) \Db_2^{-1}.
\end{align*}
We can therefore conclude our proof by (\ref{eq_axaxtildeconverge}) with a discussion similar to (\ref{eq_moderatelyly}).  Finally, when $\zeta_1$ is larger,   we can conclude our proof using a discussion similar to (\ref{eq_fastdivergentregime}) with (\ref{eq_fastonenoise}) and Lemma \ref{lem_2ndpaperresults}. Together with (\ref{eq_slowlydivergentdiagonalconcetrantion1}),   we conclude the proof.   

\subsection{Proof of Theorem \ref{thm_secenario2moderately}}
We now prove Theorem \ref{thm_secenario2moderately} when $\zeta_1 \geq \zeta_2 \geq 1.$ 
For part (1), (\ref{eq_thm35eqone}) follows from (\ref{eq_onematrixbound}) and an analogous result for $\Ab_2$ that
\begin{equation}\label{eq_anaresult}
\left\| \Ab_2-\widetilde{\Ab}_{2,s} \right\| \prec n^{-1/2},
\end{equation}
as well as the facts that $\| \Ab_1 \|=\OO_{\prec}(1), \| \Ab_2 \|=\OO_{\prec}(1).$ Second, (\ref{eq_thm35eqtwo}) follows from (\ref{eq_onematrixboundoneone}), (2) of Lemma \ref{lem_hardmard} and the assumption that $\zeta_2 \leq \zeta_1.$ 
For part (2) and (3), the proof  follows from (1) of Lemma \ref{lem_2ndpaperresults} and its counterpart for $\Ab_2.$

\subsection{Proof of Theorem \ref{thm_mainbandwidthtwo}}
To prove the results of Theorem \ref{thm_mainbandwidthtwo},
we first study the adaptive bandwidth $h_1$ and $h_2$. When $0 \leq \zeta_2<1,$ 
by Lemma \ref{lem_con} about the sub-Gaussian random vector, we have that for $i \neq j,$
\begin{equation*}
\| \yb_i-\yb_j \|^2=2(p_2+\sigma_2^2)+\OO_{\prec}(p_2^{\zeta_2}+\sqrt{p_2}).
\end{equation*}
Since $\zeta_2<1,$  $\| \yb_i-\yb_j \|_2^2$ are concentrated around $2p_2$. Then for any $\omega \in (0,1)$  and $h_2$ chosen according to (\ref{eq_hxhyadaptivechoise}), we have that
\begin{equation}\label{eq_asymptoticrelation}
 h_2 \asymp p_2.
\end{equation}
Similarly, when $0 \leq \zeta_1<1,$ we have that $ h_1 \asymp p_1.$ 
Now we can prove part (1) when $0 \leq \zeta_1<1$. Denote 
\begin{equation*}
f\left( \frac{\| \xb_i-\xb_j \|_2^2}{h_1} \right)=g_1\left( \frac{\| \xb_i-\xb_j \|_2^2}{p_1} \right), \ f(x)=\exp(-\upsilon x),
\end{equation*}
where $g_k(x):=f(p_k x/h_k), \ k=1,2.$ Since
$\frac{p_k}{h_k} \asymp 1$,
we can apply the proof of Theorem \ref{thm_senarioiionesubcritical} to the kernel functions $g_k(x), k=1,2. $ The only difference is that the constant $\upsilon$ is now replaced by $\upsilon p_k/h_k.$ When $\zeta_1 \geq 1,$ the modification is similar except that we also need to use (2) of Lemma \ref{lem_2ndpaperresults}.

The other two cases can be obtained similarly by recalling the following fact. For any $\zeta_1 \geq 1,$ let $h_1$ be the bandwidth selected using (\ref{eq_hxhyadaptivechoise}), we have that  
for some constants $C_1, C_2>0,$ with high probability
\begin{equation}
C_1(\sigma_1^2 \log^{-1}n+p_1) \leq h_1 \leq C_2 \sigma_1^2 \log^2n.
\end{equation}
Also, note that (2) of Lemma \ref{lem_2ndpaperresults} holds. See Corollary 3.2 of \cite{DW2} for the proof.
With this fact, 
for part (2), (\ref{eq_thm41eqone}) follows from \eqref{eqa}
and its counterpart for $\Ab_2$ and the fact $\| \Ab_1 \| \prec 1, \| \Ab_2 \| \prec 1;$ (\ref{eq_thm41eqtwo}) follows from (\ref{eqb}) and its counterpart for $\Ab_2$ and (2) of Lemma \ref{lem_hardmard}; (\ref{eq_thm41eqthree}) follows from (\ref{eqc}) and its counterpart for $\Ab_2$ and the assumption $\zeta_1 \geq \zeta_2.$

\bibliographystyle{abbrv}
\bibliography{sensornonnull}

\begin{thebibliography}{10}

\bibitem{BHPZ}
Z.~Bao, J.~Hu, G.~Pan, and W.~Zhou.
\newblock {Canonical correlation coefficients of high-dimensional Gaussian
  vectors: Finite rank case}.
\newblock {\em The Annals of Statistics}, 47(1):612 -- 640, 2019.

\bibitem{BELKIN20081289}
M.~Belkin and P.~Niyogi.
\newblock Towards a theoretical foundation for {L}aplacian-based manifold
  methods.
\newblock {\em Journal of Computer and System Sciences}, 74(8):1289--1308,
  2008.

\bibitem{MR3044473}
C.~Bordenave.
\newblock On {E}uclidean random matrices in high dimension.
\newblock {\em Electron. Commun. Probab.}, 18:no. 25, 8, 2013.

\bibitem{CS}
X.~Cheng and A.~Singer.
\newblock The spectrum of random inner-product kernel matrices.
\newblock {\em Random Matrices: Theory and Applications}, 02(04):1350010, 2013.

\bibitem{DJ}
X.~{Ding} and H.~C. {Ji}.
\newblock {Local laws for multiplication of random matrices and spiked
  invariant model}.
\newblock {\em arXiv preprint arXiv 2010.16083}, 2020.

\bibitem{DW2}
X.~{Ding} and H.-T. {Wu}.
\newblock {Impact of signal-to-noise ratio and bandwidth on graph Laplacian
  spectrum from high-dimensional noisy point cloud}.
\newblock {\em arXiv preprint arXiv 2011.10725}, 2020.

\bibitem{DW1}
X.~{Ding} and H.~T. {Wu}.
\newblock On the spectral property of kernel-based sensor fusion algorithms of
  high dimensional data.
\newblock {\em IEEE Transactions on Information Theory}, 67(1):640--670, 2021.

\bibitem{DV}
Y.~Do and V.~Vu.
\newblock The spectrum of random kernel matrices: Universality results for
  rough and varying kernels.
\newblock {\em Random Matrices: Theory and Applications}, 02(03):1350005, 2013.

\bibitem{7558246}
D.~{Dov}, R.~{Talmon}, and I.~{Cohen}.
\newblock Kernel-based sensor fusion with application to audio-visual voice
  activity detection.
\newblock {\em IEEE Transactions on Signal Processing}, 64(24):6406--6416,
  2016.

\bibitem{8281539}
D.~{Dov}, R.~{Talmon}, and I.~{Cohen}.
\newblock Sequential audio-visual correspondence with alternating diffusion
  kernels.
\newblock {\em IEEE Transactions on Signal Processing}, 66(12):3100--3111,
  2018.

\bibitem{dunson2019diffusion}
D.~B. Dunson, H.-T. Wu, and N.~Wu.
\newblock Diffusion based {G}aussian process regression via heat kernel
  reconstruction.
\newblock {\em Applied and Computional Harmonic Analysis}, 2021.

\bibitem{informationplusenoise}
N.~El~Karoui.
\newblock {On information plus noise kernel random matrices}.
\newblock {\em The Annals of Statistics}, 38(5):3191 -- 3216, 2010.

\bibitem{NEKkernel}
N.~El~Karoui.
\newblock {The spectrum of kernel random matrices}.
\newblock {\em The Annals of Statistics}, 38(1):1 -- 50, 2010.

\bibitem{NEKW2}
N.~El~Karoui and H.-T. Wu.
\newblock {Graph connection Laplacian and random matrices with random blocks}.
\newblock {\em Information and Inference: A Journal of the IMA}, 4(1):1--44,
  2015.

\bibitem{NEKW}
N.~El~Karoui and H.-T. Wu.
\newblock {Graph connection Laplacian methods can be made robust to noise}.
\newblock {\em The Annals of Statistics}, 44(1):346 -- 372, 2016.

\bibitem{erdos2017dynamical}
L.~Erd{\H{o}}s and H.~Yau.
\newblock {\em A Dynamical Approach to Random Matrix Theory}.
\newblock Courant Lecture Notes. American Mathematical Society, 2017.

\bibitem{MR3916104}
Z.~Fan and A.~Montanari.
\newblock The spectral norm of random inner-product kernel matrices.
\newblock {\em Probab. Theory Related Fields}, 173(1-2):27--85, 2019.

\bibitem{MR4130541}
N.~Garc\'ia~Trillos, M.~Gerlach, M.~Hein, and D.~Slepcev.
\newblock Error estimates for spectral convergence of the graph {L}aplacian on
  random geometric graphs toward the {L}aplace-{B}eltrami operator.
\newblock {\em Found. Comput. Math.}, 20(4):827--887, 2020.

\bibitem{gustafsson2012statistical}
F.~Gustafsson.
\newblock {\em Statistical Sensor Fusion}.
\newblock Professional Publishing House, 2012.

\bibitem{6788402}
D.~R. {Hardoon}, S.~{Szedmak}, and J.~{Shawe-Taylor}.
\newblock Canonical correlation analysis: An overview with application to
  learning methods.
\newblock {\em Neural Computation}, 16(12):2639--2664, 2004.

\bibitem{10.1007/11503415_32}
M.~Hein, J.-Y. Audibert, and U.~von Luxburg.
\newblock From graphs to manifolds -- weak and strong pointwise consistency of
  graph {L}aplacians.
\newblock In P.~Auer and R.~Meir, editors, {\em Learning Theory}, pages
  470--485, 2005.

\bibitem{MR2332434}
M.~Hein, J.-Y. Audibert, and U.~von Luxburg.
\newblock Graph {L}aplacians and their convergence on random neighborhood
  graphs.
\newblock {\em J. Mach. Learn. Res.}, 8:1325--1368, 2007.

\bibitem{Horst1961}
P.~Horst.
\newblock Relations among m sets of measures.
\newblock {\em Psychometrika}, 26(2):129--149, 1961.

\bibitem{hotelling1936relations}
H.~Hotelling.
\newblock Relations between two sets of variates.
\newblock {\em Biometrika}, 28:321--377, 1936.

\bibitem{Hwang2013}
H.~Hwang, K.~Jung, Y.~Takane, and T.~S. Woodward.
\newblock A unified approach to multiple-set canonical correlation analysis and
  principal components analysis.
\newblock {\em British Journal of Mathematical and Statistical Psychology},
  66(2):308--321, 2013.

\bibitem{Jipaper}
H.~C. Ji.
\newblock {Regularity Properties of Free Multiplicative Convolution on the
  Positive Line}.
\newblock {\em International Mathematics Research Notices}, 07 2020.
\newblock rnaa152.

\bibitem{johnstone2001}
I.~M. Johnstone.
\newblock On the distribution of the largest eigenvalue in principal components
  analysis.
\newblock {\em Ann. Statist.}, 29(2):295--327, 2001.

\bibitem{KR}
S.~P. Kasiviswanathan and M.~Rudelson.
\newblock Spectral norm of random kernel matrices with applications to privacy.
\newblock In {\em Approximation, Randomization, and Combinatorial Optimization.
  Algorithms and Techniques, {APPROX/RANDOM} 2015, August 24-26, 2015,
  Princeton, NJ, {USA}}, volume~40 of {\em LIPIcs}, pages 898--914. Schloss
  Dagstuhl - Leibniz-Zentrum f{\"{u}}r Informatik, 2015.

\bibitem{MR3704770}
A.~Knowles and J.~Yin.
\newblock Anisotropic local laws for random matrices.
\newblock {\em Probab. Theory Related Fields}, 169(1-2):257--352, 2017.

\bibitem{7214350}
D.~{Lahat}, T.~{Adali}, and C.~{Jutten}.
\newblock Multimodal data fusion: An overview of methods, challenges, and
  prospects.
\newblock {\em Proceedings of the IEEE}, 103(9):1449--1477, 2015.

\bibitem{LEDERMAN2018509}
R.~R. Lederman and R.~Talmon.
\newblock Learning the geometry of common latent variables using
  alternating-diffusion.
\newblock {\em Applied and Computational Harmonic Analysis}, 44(3):509--536,
  2018.

\bibitem{lederman2018learning}
R.~R. Lederman and R.~Talmon.
\newblock Learning the geometry of common latent variables using
  alternating-diffusion.
\newblock {\em Applied and Computational Harmonic Analysis}, 44(3):509--536,
  2018.

\bibitem{lindenbaum2018multiview}
O.~Lindenbaum, Y.~Bregman, N.~Rabin, and A.~Averbuch.
\newblock Multiview kernels for low-dimensional modeling of seismic events.
\newblock {\em IEEE Transactions on Geoscience and Remote Sensing},
  56(6):3300--3310, 2018.

\bibitem{lindenbaum2020multi}
O.~Lindenbaum, A.~Yeredor, M.~Salhov, and A.~Averbuch.
\newblock Multi-view diffusion maps.
\newblock {\em Information Fusion}, 55:127--149, 2020.

\bibitem{liu2020diffuse}
G.-R. Liu, Y.-L. Lo, J.~Malik, Y.-C. Sheu, and H.-T. Wu.
\newblock Diffuse to fuse eeg spectra--intrinsic geometry of sleep dynamics for
  classification.
\newblock {\em Biomedical Signal Processing and Control}, 55:101576, 2020.

\bibitem{2021arXiv210203297M}
Z.~{Ma} and F.~{Yang}.
\newblock {Sample canonical correlation coefficients of high-dimensional random
  vectors with finite rank correlations}.
\newblock {\em arXiv preprint arXiv 2102.03297}, 2021.

\bibitem{MP}
V.~A. Marchenko and L.~A. Pastur.
\newblock Distribution of eigenvalues for some sets of random matrices.
\newblock {\em Mathematics of the USSR-Sbornik}, 1(4):457--483, 1967.

\bibitem{marshall2018time}
N.~F. Marshall and M.~J. Hirn.
\newblock Time coupled diffusion maps.
\newblock {\em Applied and Computational Harmonic Analysis}, 45(3):709--728,
  2018.

\bibitem{NCCApaper}
T.~Michaeli, W.~Wang, and K.~Livescu.
\newblock Nonparametric canonical correlation analysis.
\newblock In {\em Proceedings of the 33rd International Conference on
  International Conference on Machine Learning - Volume 48}, ICML'16, page
  1967–1976, 2016.

\bibitem{PY}
N.~S. Pillai and J.~Yin.
\newblock {Universality of covariance matrices}.
\newblock {\em The Annals of Applied Probability}, 24(3):935 -- 1001, 2014.

\bibitem{samuels2007brain}
M.~A. Samuels.
\newblock The brain--heart connection.
\newblock {\em Circulation}, 116(1):77--84, 2007.

\bibitem{MR4010764}
T.~Shnitzer, M.~Ben-Chen, L.~Guibas, R.~Talmon, and H.-T. Wu.
\newblock Recovering hidden components in multimodal data with composite
  diffusion operators.
\newblock {\em SIAM J. Math. Data Sci.}, 1(3):588--616, 2019.

\bibitem{shnitzer2019recovering}
T.~Shnitzer, M.~Ben-Chen, L.~Guibas, R.~Talmon, and H.-T. Wu.
\newblock Recovering hidden components in multimodal data with composite
  diffusion operators.
\newblock {\em SIAM Journal on Mathematics of Data Science}, 1(3):588--616,
  2019.

\bibitem{SINGER2006128}
A.~Singer.
\newblock From graph to manifold laplacian: The convergence rate.
\newblock {\em Applied and Computational Harmonic Analysis}, 21(1):128--134,
  2006.

\bibitem{TALMON2019848}
R.~Talmon and H.-T. Wu.
\newblock Latent common manifold learning with alternating diffusion: Analysis
  and applications.
\newblock {\em Applied and Computational Harmonic Analysis}, 47(3):848--892,
  2019.

\bibitem{talmon2019latent}
R.~Talmon and H.-T. Wu.
\newblock Latent common manifold learning with alternating diffusion: analysis
  and applications.
\newblock {\em Applied and Computational Harmonic Analysis}, 47(3):848--892,
  2019.

\bibitem{DANORIGINALPAPER}
D.~Voiculescu.
\newblock Multiplication of certain non-commuting random variables.
\newblock {\em Journal of Operator Theory}, 18(2):223--235, 1987.

\bibitem{xiao2019manifold}
L.~Xiao, J.~M. Stephen, T.~W. Wilson, V.~D. Calhoun, and Y.-P. Wang.
\newblock A manifold regularized multi-task learning model for iq prediction
  from two fmri paradigms.
\newblock {\em IEEE Transactions on Biomedical Engineering}, 67(3):796--806,
  2019.

\bibitem{ZHAO201743}
J.~Zhao, X.~Xie, X.~Xu, and S.~Sun.
\newblock Multi-view learning overview: Recent progress and new challenges.
\newblock {\em Information Fusion}, 38:43--54, 2017.

\bibitem{zhuang2020technical}
X.~Zhuang, Z.~Yang, and D.~Cordes.
\newblock A technical review of canonical correlation analysis for neuroscience
  applications.
\newblock {\em Human Brain Mapping}, 41(13):3807--3833, 2020.

\end{thebibliography}

\end{document}